\documentclass{article}

\PassOptionsToPackage{numbers, compress}{natbib}

\usepackage[final]{neurips_2021}
\usepackage[pagebackref=true,breaklinks=true,letterpaper=true,colorlinks,bookmarks=false]{hyperref}

\usepackage[utf8]{inputenc} 
\usepackage[T1]{fontenc}    
\usepackage{hyperref}       
\usepackage{url}            
\usepackage{cellspace}       
\usepackage{booktabs}       
\usepackage{amsfonts}       
\usepackage{nicefrac}       
\usepackage{microtype}      
\usepackage{adjustbox}

\usepackage{wrapfig}
\usepackage{xspace}
\usepackage{graphicx}
\usepackage{caption}
\usepackage{subcaption} 
\usepackage{algorithm}

\usepackage{array}
\usepackage[dvipsnames]{color}
\usepackage{xcolor}
\usepackage{algpseudocode}
\usepackage{pifont}
\usepackage{tikz}

\definecolor{darkgreen}{RGB}{5,110,49}

\usepackage{amsthm}
\usepackage{thmtools, thm-restate}

\usepackage{amssymb}
\usepackage{amsmath,amsfonts}
\usepackage{amsopn}
\usepackage{bm} 
\usepackage{multirow}

\usepackage{mathtools}

\newcommand{\defeq}{\vcentcolon=}

\newcommand{\vct}[1]{\boldsymbol{#1}} 

\newcommand{\field}[1]{\mathbb{#1}}
\newcommand{\R}{\field{R}} 
\newcommand{\T}{^{\textrm T}} 
\newcommand{\RSA}{\mathbb{R^{\mathcal{S} \times \mathcal{A}}}}
\newcommand{\RE}{\overline{\mathbb{R}}}

\newcommand{\norm}[1]{\|#1\|}

\newcommand{\ProbOpr}[1]{\mathbb{#1}}

\newcommand{\expect}[2]{%
\ifthenelse{\equal{#2}{}}{\ProbOpr{E}_{#1}}
{\ifthenelse{\equal{#1}{}}{\ProbOpr{E}\left[#2\right]}{\ProbOpr{E}_{#1}\left[#2\right]}}} 
\newcommand{\var}[2]{%
\ifthenelse{\equal{#2}{}}{\ProbOpr{VAR}_{#1}}
{\ifthenelse{\equal{#1}{}}{\ProbOpr{VAR}\left[#2\right]}{\ProbOpr{VAR}_{#1}\left[#2\right]}}} 

\DeclareMathOperator*{\argmax}{\mathrm{argmax}}
\DeclareMathOperator*{\argmin}{\mathrm{argmin}}

\newcommand{\sign}{\operatornamewithlimits{sign}}

\newcommand{\vr}{\vct{r}}

\newcommand{\eat}[1]{}

\newcommand{\states}{\mathcal{S}}
\newcommand{\actions}{\mathcal{A}}
\newcommand{\E}{\mathbb{E}}
\newcommand{\J}{\mathcal{J}}
\renewcommand{\T}{\mathcal{T}}
\newcommand{\sm}{\operatorname*{softmax}}

\DeclareMathOperator{\stopgrad}{stop\_grad}

\newcommand{\pix}{\kern 0.1em}
\newcommand{\pmm}{\kern 0.35em$\pm$\kern 0.35em}

\newtheorem{theorem}{Theorem}[section]
\newtheorem{lemma}[theorem]{Lemma}
\newtheorem{corollary}{Corollary}[theorem]
\newtheorem{prop}[theorem]{Proposition}

\definecolor{faintgray}{gray}{0.8}
\newcommand{\cm}{\ding{51}}
\newcommand{\xm}{{\color{faintgray}$\times$ }}

\title{IQ-Learn: Inverse soft-Q Learning for Imitation}

\author{Divyansh Garg$^{1}$ \hspace{10pt} Shuvam Chakraborty$^{1}$
	\hspace{10pt}Chris Cundy$^{1}$ \\ 
	\textbf{Jiaming Song$^{1}$ \hspace{10pt}Matthieu Geist$^{2}$ \hspace{10pt} Stefano Ermon$^{1}$}\\
	$^1$ Stanford University \hspace{10pt} 
	$^2$ Google Research, Brain Team \\
	{\tt\small \{divgarg, shuvamc, cundy, tsong, ermon\}@stanford.edu}
}

\begin{document}

\maketitle

\begin{abstract}
In many sequential decision-making problems (e.g., robotics control, game playing, sequential prediction),
human or expert data is available containing useful information about the task. 
However, imitation learning (IL) from a small amount of expert data can be challenging in high-dimensional environments with complex dynamics.
Behavioral cloning is a simple method that is widely used due to its simplicity of implementation and stable convergence but doesn't utilize any information involving the environment’s dynamics. 
Many existing methods that exploit dynamics information are difficult to train in practice due to an adversarial optimization process over reward and policy approximators or biased, high variance gradient estimators. 
We introduce a method for dynamics-aware IL which avoids adversarial training by learning a single Q-function, implicitly representing both reward and policy.
On standard benchmarks, the implicitly learned rewards show a high positive correlation with the ground-truth rewards, illustrating our method can also be used for inverse reinforcement learning (IRL). 
Our method, Inverse soft-Q learning (IQ-Learn) obtains \emph{state-of-the-art results} in offline 
and online imitation learning settings, significantly outperforming existing methods both in the number of required environment interactions and scalability in high-dimensional spaces, often by more than \textbf{3x}\footnote{Our implementation is available at \url{https://github.com/Div99/IQ-Learn}.}.
\end{abstract}


\section{Introduction}
Imitation of an expert has long been recognized as a powerful approach for
sequential decision-making~\cite{ng2000algorithms, abbeel2004apprenticeship}, with applications as diverse as healthcare~\cite{wang2020adversarial}, autonomous driving~\cite{zhou2021exploring}, and playing complex strategic games~\cite{deepmind2019mastering}.
In the imitation learning (IL) setting, we are given a set of expert trajectories, with the goal of learning a policy which induces behavior similar to the expert's. 
The learner has no access to the reward, and no explicit knowledge of the dynamics. 

The simple behavioural cloning~\cite{ross2010efficient} approach simply maximizes the probability of the expert's actions under the learned policy, approaching the IL problem as a supervised learning problem. 
While this can work well in simple environments and with large quantities of data, it ignores the sequential nature of the decision-making problem, and small errors can quickly compound when the learned policy departs from the states observed under the expert. A natural way of introducing the environment dynamics is by framing the IL problem as an Inverse RL (IRL) problem, aiming to learn a reward function under which the expert's trajectory is optimal, and from which the learned imitation policy can be trained~\cite{abbeel2004apprenticeship}. 
This framing has inspired several approaches which use rewards either explicitly or implicitly to incorporate dynamics while learning an imitation policy~\cite{Ho2016GenerativeAI, fu2018learning, Reddy2020SQILIL, Kostrikov2020Imitation}.
However, these dynamics-aware methods are typically hard to put into practice due to unstable learning which can be sensitive to hyperparameter choice or minor implementation details~\cite{kostrikov2018discriminator}.


In this work, we introduce a dynamics-aware imitation learning method which has stable, non-adversarial training, allowing us to achieve state-of-the-art performance on imitation learning benchmarks. 
Our key insight is that much of the difficulty with previous IL methods arises from the IRL-motivated representation of the IL problem as a min-max problem over reward and policy~\cite{Ho2016GenerativeAI, abbeel2004apprenticeship}. 

This introduces a requirement to separately model the reward and policy, and train these two functions jointly, often in an adversarial fashion. 
Drawing on connections between RL and energy-based models~\cite{Haarnoja2018SoftAO, haarnoja2017reinforcement}, we propose learning a \emph{single model for the $Q$-value}. 
The $Q$-value then implicitly defines both a reward and policy function. This turns a difficult min-max problem over policy and reward functions into a simpler minimization problem over a single function, the \(Q\)-value. 
Since our problem has a one-to-one correspondence with the min-max problem studied in adversarial IL \citep{Ho2016GenerativeAI}, we 
maintain the generality and guarantees of these previous approaches, resulting in a meaningful reward that may be used for inverse reinforcement learning. Furthermore, our method may be used to minimize a variety of statistical divergences between the expert and learned policy. 
We show that we recover several previously-described approaches as special cases of particular divergences, such as the regularized behavioural cloning of~\cite{piot2014boosted}, and the conservative Q-learning of~\cite{kumar2020conservative}.


In our experiments, we find that our method is performant even with very sparse data - surpassing prior methods using \textit{one expert demonstration} in the completely offline setting -  and can scale to complex image-based tasks like Atari reaching expert performance. Moreover, our learnt rewards are highly predictive of the original environment rewards.
Finally, our method is robust to distribution shifts in the environment showing great generalization performance to never seen goals and an ability to act as a meta-learner.

Concretely, our contributions are as follows:
\begin{itemize}
    \item We present a modified $Q$-learning update rule for imitation learning that can be implemented on top of soft-Q learning or soft actor-critic (SAC) algorithms in fewer than \textbf{15} lines of code.
    
    \item We introduce a simple framework to minimize a wide range of statistical distances: Integral Probability Metrics (IPMs) and f-divergences, between the expert and learned distributions.
    
    \item We empirically show state-of-art results in a variety of imitation learning settings: online and offline IL. On the complex Atari suite, we outperform prior methods by \textbf{3-7x} while requiring \textbf{3x} less environment steps.
    
    \item We characterize our learnt rewards and show a high positive correlation with the ground-truth rewards, justifying the use of our method for Inverse Reinforcement Learning. 
\end{itemize}

\begin{table}[t]\small
\newcolumntype{O}{>{          \arraybackslash}m{0.4 cm}}
\newcolumntype{A}{>{          \arraybackslash}m{2.8 cm}}
\newcolumntype{B}{>{\centering\arraybackslash}m{1.85cm}}
\newcolumntype{C}{>{\centering\arraybackslash}m{1.9 cm}}
\newcolumntype{D}{>{\centering\arraybackslash}m{2.1 cm}}
\newcolumntype{E}{>{\centering\arraybackslash}m{1.8 cm}}
\newcolumntype{F}{>{\centering\arraybackslash}m{2.1 cm}}
\newcolumntype{G}{>{\centering\arraybackslash}m{1.7 cm}}
\setlength\tabcolsep{0pt}
\renewcommand{\arraystretch}{0.93}
\renewcommand\T{\rule{0pt}{3.4ex}}       
\newcommand\B{\rule[-0.9ex]{0pt}{0pt}} 

\newcommand{\midsepremove}{\aboverulesep = 0mm \belowrulesep = 0mm}
\newcommand{\midsepdefault}{\aboverulesep = 0.605mm \belowrulesep = 0.984mm}


\vspace{-1.5em}
\caption{A comparison of various algorithms for imitation learning. ``Convergence Guarantees'' refers to if a proof is given that the algorithm converges to the correct policy with sufficient data. We consider an algorithm “directly optimized” if it consists of an optimization algorithm (such as gradient descent) applied to the parameters of a single function    
}
\label{tab:related}
\begin{center}
\begin{adjustbox}{max width=\textwidth}
\begin{tabular}{O|ABCDEFG}
\toprule
  \multicolumn{2}{c}{\textbf{Method}} & \textbf{Reference}
  & {Dynamics \smash{$\hphantom{^{X}}$}Aware\smash{\pix}}
  & {Non-Adversarial \smash{$\hphantom{^{X}}$}Training\smash{\pix}}
  & {Convergence \smash{$\hphantom{^{3}}$}Guarantees\smash{\pix}}
  & {Non-restrictive \smash{$\hphantom{^{X}}$}Reward\smash{\pix}}
  & {Direct \smash{$\hphantom{^{X}}$}Optimization\smash{\pix}} \\
\midrule
\parbox[t]{2mm}{\multirow{3}{*}{\pix\rotatebox[origin=c]{90}{\textbf{Online} \hspace{2.5em}}}}
& ~~Max Margin IRL            & \cite{ng2000algorithms, abbeel2004apprenticeship}      & \cm & \cm & \cm & \xm & \xm \\
& ~~Max Entropy IRL           & \cite{Ziebart2008MaximumEI}         & \cm & \cm & \cm & \xm & \xm \\
& ~~GAIL/AIRL & \cite{Ho2016GenerativeAI, fu2018learning}& \cm & \xm & \cm & \cm & \xm \\
& ~~ASAF & \cite{barde2020adversarial}& \cm & \cm & \cm & \xm & \cm \\
& ~~SQIL & \cite{Reddy2020SQILIL}& \cm & \cm & \xm & \xm & \cm \B \\ 
\cmidrule{2-8}
& ~\textbf{ Ours (Online)} & -- & \cm & \cm & \cm & \cm &\cm \\
\midrule 
  \parbox[t]{2mm}{\multirow{3}{*}{\pix\rotatebox[origin=c]{90}{{\textbf {Offline \hspace{3em}}}}}}
& ~~Max Margin IRL            & \cite{lee2019truly, klein2011batch}                    & \cm & \cm & \cm & \xm & \xm \\
& ~~Max Likelihood IRL        & \cite{jain2019model}                                   & \cm & \cm & \cm & \xm & \xm \\
& ~~Max Entropy IRL           & \cite{herman2016inverse}                               & \cm & \cm & \cm & \xm & \xm \\
& ~~ValueDICE & \cite{Kostrikov2020Imitation}                          & \cm & \xm & \xm & \xm & \xm \\
& ~~Behavioral Cloning & \cite{ross2010efficient} & \xm &\cm & \cm & \xm & \cm \\
& ~~Regularized BC & \cite{piot2014boosted} & \cm &\cm & \cm & \xm & \cm \\
& ~~EDM & \cite{jarrett2020strictly} & \cm &\cm & \xm & \cm & \cm \B \\
\cmidrule{2-8}
\midsepdefault
& ~~\textbf{Ours (Offline)} & -- & \cm & \cm & \cm & \cm &\cm\\
\bottomrule 
\end{tabular}

\end{adjustbox}
\end{center}
\vspace{-1.75em}
\end{table}

\section{Background}
\paragraph{Preliminaries}  We consider environments represented as a Markov decision process (MDP), which is  defined by a tuple ($\mathcal{S}, \mathcal{A}, p_0, \mathcal{P}, r, \gamma)$.
$\mathcal{S}, \mathcal{A}$ represent state and action spaces, $p_0$ and $\mathcal{P}(s'|s, a)$ represent the initial state distribution and the dynamics, $r(s, a) \in \mathcal{R}$ represents the reward
function, and $\gamma \in (0, 1)$ represents the discount factor. $\mathbb{R}^{\mathcal{S} \times \mathcal{A}}=\{x: \mathcal{S} \times \mathcal{A} \rightarrow \mathbb{R}\}$ will denote the set of all functions in the state-action space and $\overline{\mathbb{R}}$ will denote the extended real numbers $\mathbb{R} \cup \{\infty\}$.  Sections~\ref{sec:IQLearn} and~\ref{sec:approach}
will work with
finite state and action spaces $\mathcal{S}$ and $\mathcal{A}$, but our algorithms and experiments later in the
paper use continuous environments. $\Pi$ is the set of all stationary stochastic
policies that take actions in $\mathcal{A}$ given states in $\mathcal{S}$. 
We work in the $\gamma$-discounted infinite horizon setting, and we will use an expectation with respect to a policy $\pi \in \Pi$ to denote an expectation with respect to the trajectory it generates: 
$\mathbb{E}_{\pi}[r(s, a)] \triangleq \mathbb{E} [\sum_{t=0}^\infty \gamma^{t} r(s_t, a_t)]$, where $s_0 \sim p_0$, $a_t \sim \pi(\cdot|s_t)$, and $s_{t+1} \sim  \mathcal{P}(\cdot|s_t, a_t)$ for $t \geq 0$. For a policy $\pi \in \Pi$, we define its occupancy measure $\rho_{\pi}: \mathcal{S} \times \mathcal{A} \rightarrow \mathbb{R}$ as $\rho_{\pi}(s, a)=(1-\gamma)\pi(a | s) \sum_{t=0}^{\infty} \gamma^{t} P\left(s_{t}=s | \pi\right)$. 
We refer to the expert policy as $\pi_E$ and its occupancy measure as $\rho_E$.
In practice, $\pi_E$ is unknown and we have access to a sampled dataset of demonstrations $\mathcal{D}$. 
For brevity, we refer to $\rho_\pi$ as $\rho$ for a learnt policy in the paper.
\paragraph{Soft $Q$-functions} For a reward $r \in \mathcal{R}$ and $\pi \in \Pi$, the soft Bellman operator
$\mathcal{B}^\pi_r: \mathbb{R^{\mathcal{S} \times \mathcal{A}}} \rightarrow  \mathbb{R^{\mathcal{S} \times \mathcal{A}}}$ is defined as
$
    (\mathcal{B}^\pi_r Q)(s, a) = r(s, a) + \gamma \mathbb{E}_{s' \sim \mathcal{P}(s,a)}V^\pi(s')
$ 
with $V^\pi(s) = \mathbb{E}_{a \sim \pi(\cdot|s)} \left[Q(s, a) - \log\pi(a |s) \right]$. The soft Bellman operator
is contractive~\cite{Haarnoja2018SoftAO} and defines a unique soft $Q$-function for $r$, given as the fixed point solution $Q=\mathcal{B}^\pi_r Q$ with $Q \in \Omega$.

\paragraph{Max Entropy Reinforcement Learning}

For a given reward function $r \in \mathcal{R} $, maximum entropy RL \cite{haarnoja2017reinforcement, Bloem2014InfiniteTH} aims to learn a policy that maximizes the expected cumulative discounted reward along with the entropy in each state: $\max _{\pi \in \Pi} \mathbb{E}_{\rho_\pi}[r(s, a)] + H(\pi)$, where $H(\pi) \triangleq \mathbb{E}_{\rho_\pi}[-\log \pi(a|s)]$ is the discounted causal entropy of the policy $\pi$.
The optimal policy satisfies  \cite{ziebart2010modeling, Bloem2014InfiniteTH}:
\begin{align}
\label{eq:pi}
 \pi^*(a|s) = \frac{1}{Z_s}\exp{(Q^*(s, a))},
\end{align}
where $Z_s$ is the normalization factor given as $\sum_{a^{\prime}} \exp \left(Q^*\left(s, a^{\prime}\right)\right)$ and $Q^*$ is the optimal soft $Q$-function.

$Q^*$ satisfies the soft-Bellman equation:
\begin{equation}
\label{eq:q}
Q^*(s, a) = (\mathcal{B}^* Q^*)(s,a) :=  r(s, a) + \gamma \mathbb{E}_{s' \sim \mathcal{P}(\cdot | s, a)} \Big[\log \sum_{a'} \exp(Q^*(s', a'))\Big].
\end{equation}
In continuous action spaces, $Z_s$ becomes computationally intractable
and soft actor-critic methods like SAC~\cite{Haarnoja2018SoftAO} can be used to learn an explicit policy.

\paragraph{Max Entropy Inverse Reinforcement Learning}
Given demonstrations sampled using the policy $\pi_E$, maximum entropy Inverse RL aims to recover the reward function in a family of functions $\mathcal{R} \subset \mathbb{R}^{\mathcal{S} \times \mathcal{A}}$  that rationalizes the expert behavior by solving the optimization problem: $
\max_{r \in \mathcal{R}} \min _{\pi \in \Pi} \mathbb{E}_{\rho_{E}}[r(s, a)]  - \left( \mathbb{E}_{\rho_\pi}[r(s, a)] + H(\pi) \right)$, where the expected reward of $\pi_E$ is empirically approximated using a dataset $\mathcal{D}$.
It looks for a reward function that assigns high reward to the expert policy and low reward to other ones, while searching for the best policy for the reward function in an inner loop.

The Inverse RL objective can be generalized
in terms of its occupancy measure, and with a convex reward regularizer $\psi: \mathbb{R}^{\mathcal{S} \times \mathcal{A}} \rightarrow \overline{\mathbb{R}}$ \cite{Ho2016GenerativeAI}
\begin{align}
\label{eq:irl}
 \underset{r \in \mathcal{R}}{\max} \min _{\pi \in \Pi} L(\pi, r) = \mathbb{E}_{\rho_{E}}[r(s, a)]  - \mathbb{E}_{\rho_\pi}[r(s, a)] - H(\pi) -\psi(r).
\end{align}

In general, for a non-restrictive set of reward functions $\mathcal{R} = \mathbb{R}^{\mathcal{S} \times \mathcal{A}}$, we can exchange the max-min resulting in an objective that minimizes the statistical distance parameterized by $\psi$, between the expert and the policy \cite{Ho2016GenerativeAI}
\begin{align}
\label{eq:orig_irl}
 \min _{\pi \in \Pi} \underset{r \in \mathcal{R}}{\max} \ L(\pi, r) = \min _{\pi \in \Pi} d_\psi(\rho_\pi, \rho_{E} ) - H(\pi),
\end{align}
with $d_\psi \triangleq \psi^*(\rho_E - \rho_\pi)$, where $\psi^*$ is the convex conjugate of $\psi$.



\section{Inverse soft Q-learning (IQ-Learn) Framework}
\label{sec:IQLearn}
A naive solution to the nested min-max IRL problem in (Eq. \ref{eq:irl}) involves (1) an outer loop learning rewards and (2) executing RL in an inner loop to find an optimal policy for them. However, we  know that this optimal policy can be obtained solely in terms of the soft $Q$-function (Eq. \ref{eq:pi}). Interestingly, as we will show later, the rewards can also be represented in terms of only $Q$ (Eq. \ref{eq:q}). 
Together, these observations suggest it might be possible to directly solve the IRL problem by  optimizing only over the $Q$-function, thus reducing the nested min-max problem to a single minimization problem over $Q$. 

To motivate the search of an imitation learning algorithm that depends only on the $Q$-function,  we characterize the space of $Q$-functions and policies obtained using Inverse RL. 
We will study $\pi \in \Pi$, $r \in \mathcal{R}$ and $Q$-functions $Q \in \Omega$, with fully general classes $\mathcal{R} = \Omega = \mathbb{R}^{\mathcal{S} \times \mathcal{A}} $. 
The full policy class $\Pi$ is convex, compact with $\pi_E \in \Pi$.


We start with the analysis developed in \cite{Ho2016GenerativeAI}: 
\begin{prop}
\label{lemma:3.0}
The regularized IRL objective $L({\pi, r})$ given by Eq. \ref{eq:irl}
is convex in the occupancy measure of the policy ($\rho_\pi$) and concave in the reward function (r), and for a strongly convex regularizer $\psi$ has a unique saddle point $(\pi^*, r^*)$.
\end{prop}


To characterize the $Q$-functions obtained using Inverse RL it is useful to transform the IRL problem over rewards to a problem over $Q$-functions.

Define the inverse soft Bellman operator 
$\mathcal{T}^\pi: \mathbb{R^{\mathcal{S} \times \mathcal{A}}} \rightarrow  \mathbb{R^{\mathcal{S} \times \mathcal{A}}}$ as 
\begin{align*}
    (\mathcal{T}^\pi Q)(s, a) = Q(s, a) - \gamma \mathbb{E}_{s' \sim \mathcal{P}(\cdot|s,a)}V^\pi(s').
\end{align*}
with $V^\pi(s) = \mathbb{E}_{a \sim \pi(\cdot|s)} \left[Q(s, a) - \log\pi(a |s) \right]$ as defined before. Then,
$\mathcal{T}^\pi$ inverts the soft Bellman operator $\mathcal{B}^\pi$ to map from $Q$-functions to rewards. We can get a one-to-one correspondence between $r$ and $Q$:

\begin{lemma}
\label{lemma:3.1}


For a fixed policy $\pi$, the inverse soft Bellman operator $\mathcal{T}^\pi$ is bijective, and for any $r \in \mathcal{R}$,  $Q=(\mathcal{T}^\pi)^{-1} r$ is the unique fixed point of the Bellman operator $\mathcal{B}^\pi_r$.
\end{lemma}

The proof of this lemma is in Appendix~\ref{subappx:proofs_1}. For a policy $\pi$, we are justified in changing between rewards and the corresponding soft-Q functions using $\mathcal{T}^\pi$.
Thus, we can freely transform functions from the reward-policy space, $\Pi \times \mathcal{R}$, to the $Q$-policy space, $\Pi \times \Omega$, giving us the following lemma:

\begin{lemma}
\label{lemma:3.2}
Let $L(\pi, r) = \mathbb{E}_{\rho_{E}}[r(s, a)]  - \mathbb{E}_{\rho_\pi}[r(s, a)] - H(\pi) -\psi(r)$ and \\ $\mathcal{J}(\pi, Q)= \mathbb{E}_{\rho_{E}}[(\mathcal{T}^\pi Q)(s, a)]  - \mathbb{E}_{\rho_\pi}[(\mathcal{T}^\pi Q)(s, a)] - H(\pi) -\psi(\mathcal{T}^\pi Q)$, then for all policies $\pi \in \Pi$,
$$
L({\pi, r}) = \mathcal{J}(\pi, (\mathcal{T}^\pi)^{-1}r) \ \forall r \in \mathcal{R}, \text{ and } \mathcal{J}(\pi, Q)=L({\pi, \mathcal{T}^\pi Q}) \ \forall Q \in \Omega.
$$
\end{lemma}
The proof follows directly from Lemma \ref{lemma:3.1}.
These lemmas allow us to adapt the Inverse RL objective $L({\pi, r})$ to learning $Q$ through $\mathcal{J}(\pi, Q)$, i.e., working in the $Q$-policy space.

We can  simplify the new objective $\mathcal{J}(\pi, Q)$ by working with initial states $s_0$ sampled from the initial state distribution $p_0(s)$ (Lemma~\ref{lemma:telescopic} in Appendix) as follows: 
\begin{equation}
\label{eq:psi_irl}
\mathcal{J}(\pi, Q) = \mathbb{E}_{(s, a) \sim \rho_{E}}[Q(s,a) - \gamma \mathbb{E}_{s' \sim \mathcal{P}(\cdot | s,a)}V^\pi(s')] 
- (1- \gamma) \mathbb{E}_{s_0 \sim p_0}[V^\pi(s_0)] -  \psi(\mathcal{T}^\pi Q).
\end{equation}

where again $V^\pi(s) = \mathbb{E}_{a \sim \pi(\cdot|s)} \left[Q(s, a) - \log\pi(a |s) \right]$.

We are now ready to study $\mathcal{J}(\pi, Q)$, the Inverse RL problem in the $Q$-policy space. As the regularizer $\psi$ depends on both $Q$ and $\pi$, a general analysis over all functions in $\mathbb{R^{\mathcal{S} \times \mathcal{A}}}$ becomes too difficult. We restrict ourselves to regularizers induced by a convex function $g: \R \rightarrow \RE$ such that
\begin{align}
\label{eq:psi}
\psi_{g}(r)=
\mathbb{E}_{\rho_{E}} [{g}(r(s, a))].
\end{align}
This allows us to simplify our analysis to the set of all real functions while retaining generality\footnote{Averaging over the expert occupancy allows $\psi$ to adjust to arbitrary experts and accommodate multimodality.}. We further motivate this choice in Section~\ref{sec:approach}.

\begin{prop}
\label{prop:1}
In the Q-policy space, there exists a unique saddle point $(\pi^*, Q^*)$ that optimizes $\mathcal{J}$. i.e. $Q^* = {\argmax}_{Q \in \Omega} \min _{\pi \in \Pi} \  \mathcal{J}(\pi, Q)$ and $\pi^* = \argmin _{\pi \in \Pi} {\max}_{Q \in \Omega} \  \mathcal{J}(\pi, Q)$. Furthermore, $\pi^*$ and $r^* = \mathcal{T}^{\pi^*} Q^*$ are the solution to the Inverse RL objective $L(\pi, r)$.
\end{prop}


Thus we have,
 ${\max}_{Q \in \Omega} \min _{\pi \in \Pi} \  \mathcal{J}(\pi, Q) = {\max}_{r \in \mathcal{R}} \min _{\pi \in \Pi} \  L(\pi, r)$. And the maxima $Q^*$ is simply the optimal soft Q-function for the reward $r^*$.

This tells us, even after transforming to $Q$-functions we have retained the saddle point property of the original IRL objective and optimizing $\mathcal{J}(\pi, Q)$ recovers this saddle point.
In the $Q$-policy space, we can get an additional property:

\begin{prop}
\label{prop:2}
For a fixed $Q$, $\argmin _{\pi \in \Pi} \ \mathcal{J}(\pi, Q)$ is simply the solution to max entropy RL with rewards $r = \mathcal{T}^\pi Q$. Thus, using Eq.~\ref{eq:pi}, the argmin policy satisfies
\begin{align*}
\pi_Q(a|s) = \frac{1}{Z_s}\exp(Q(s, a)),
\end{align*}
with normalization factor $Z_s = \sum_a \exp{Q(s, a)}$. Thus, the policy minima for a given $Q$ describes a manifold in the Q-policy space (Figure~\ref{fig:q_space}).
\end{prop}

Proposition \ref{prop:1} and \ref{prop:2} tell us that if we know $Q$, then the inner optimization problem in terms of policy is trivial, and obtained in a closed form! Thus, we can recover an objective that only requires learning $Q$: 
\begin{equation}
 \underset{Q \in \Omega}{\max} \min _{\pi \in \Pi} \  \mathcal{J}(\pi, Q) = \underset{Q \in \Omega}{\max} \  \mathcal{J}\left(\pi_Q, Q\right)
\end{equation}
Lastly, we have:
\begin{prop}
\label{prop:3}
Let $\mathcal{J^*}(Q) = \mathcal{J}\left(\pi_Q, Q\right)$. Then the new objective $\mathcal{J^*}$ is concave in $Q$.
\end{prop}

Thus, this new optimization objective is well-behaved and has a unique maxima $Q^*$ that gives the required saddle point as $(\pi_{Q^*}, Q^*)$. 

\begin{wrapfigure}{r}{0.7\textwidth}
\vskip -10pt
\hspace*{-0.1cm}\includegraphics[width=0.72\textwidth]{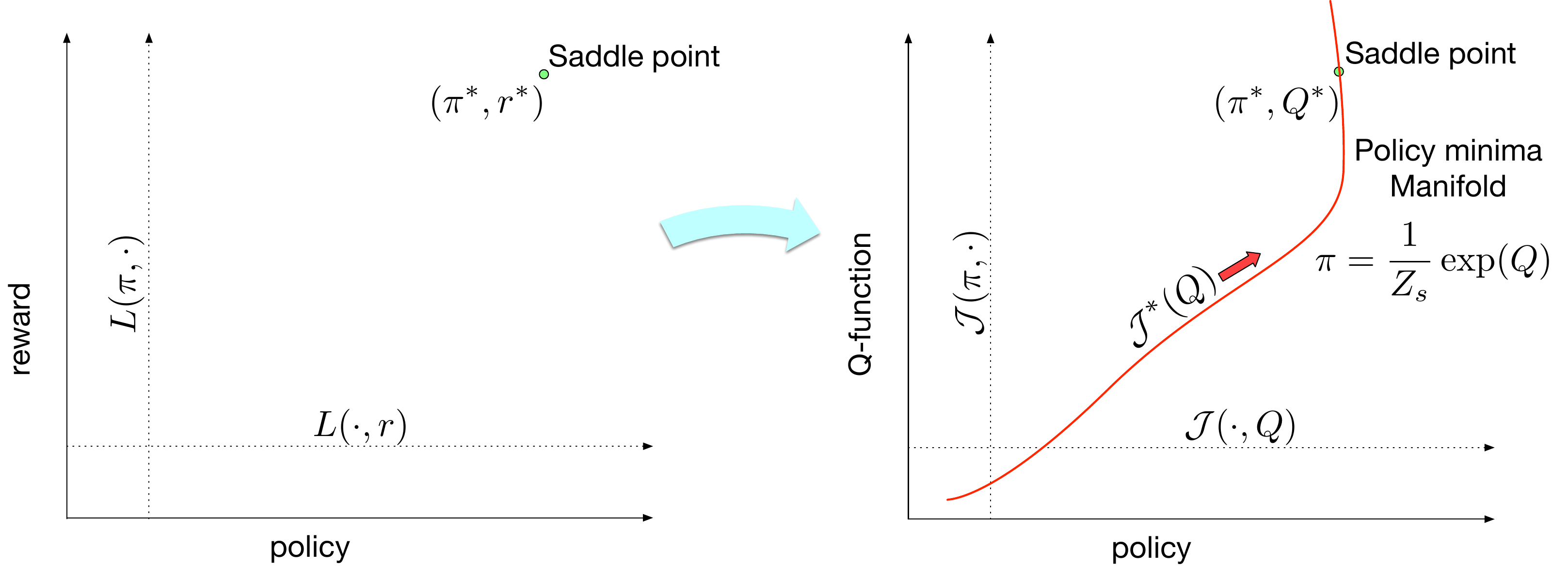}
\vskip -5pt
\caption{\small{Properties of IRL objective in reward-policy space and Q-policy space.}}
\label{fig:q_space}
\vskip -10pt
\end{wrapfigure}

In Appendix~\ref{Appx:C} we expand on our analysis and characterize the behavior for different choices of regularizer $\psi$, while giving proofs of all our propositions. Figure~\ref{fig:q_space} summarizes the properties for the IRL objective: there exists an optimal policy manifold depending on $Q$, allowing optimization along it (using $\mathcal{J^*}$) to converge to the saddle point.
We further present analysis of IL methods that learn $Q$-functions like SQIL~\cite{Reddy2020SQILIL} and ValueDICE~\cite{Kostrikov2020Imitation} and find subtle fallacies affecting their learning. 

Note that although the same analysis holds in the reward-policy space, the optimal policy manifold depends on $Q$, which isn't trivially known unlike when we work directly in the Q-policy space.

\section{Approach}
\label{sec:approach}
In this section, we develop our inverse soft-Q learning (IQ-Learn) algorithm, such that it recovers the optimal soft $Q$-function for an MDP from a given expert distribution. We start by learning energy-based models for the policy similar to soft $Q$-learning and later learn an explicit policy similar to actor-critic methods.

\subsection{General Inverse RL Objective}
For designing a practical algorithm using regularizers of the form $\psi_g$ (from Eq.~\ref{eq:psi}), we define $g$ using a concave function $\phi: \mathcal{R}_\psi \rightarrow \mathbb{R}$, such that $
\begin{array}{l}
g(x)=\left\{\begin{array}{ll}
x - \phi(x) & \text { if } x \in \mathcal{R_\psi} \\
+\infty & \text { otherwise }
\end{array}\right.
\end{array} \\
$ with the rewards constrained in $R_\psi$.

For this choice of $\psi$, the Inverse RL objective  $L(\pi, r)$ takes the form of  Eq.~\ref{eq:orig_irl} with a distance measure:
\begin{align}
\label{eq:gen_dist}
d_\psi(\rho, \rho_{E} ) = \underset{r \in \mathcal{R_\psi}}{\max}\mathbb{E}_{\rho_{E}}[\phi(r(s, a))]  - \mathbb{E}_{\rho}[r(s, a)],
\end{align}
This forms a general learning objective that allows the use of a wide-range of statistical distances including Integral Probability Metrics (IPMs) and f-divergences (see Appendix \ref{appx:B}).\footnote{We recover IPMs when using identity $\phi$ and restricted reward family ${\cal R}$.}


\subsection{Choice of Statistical Distances}
While choosing a practical regularizer, it can be useful to obtain certain properties on the reward functions we recover. 
Some (natural) nice properties are: having rewards bounded in a range, learning smooth functions or enforcing a norm-penalty.

In fact, we find these properties correspond to the Total Variation distance, the Wasserstein-1 distance and the $\chi^2$-divergence respectively. The regularizers and the induced statistical distances are summarized in Table~\ref{tbl:regularizer}:
\begin{table}[h]
	\vskip-18pt
    \small
	\centering
	\caption{\small {Enforced reward property, corresponding regularizer $\psi$ and statistical distance ($R_{\text{max}}, K, \alpha \in \mathbb{R}^+$ )}}   %
	\label{tbl:regularizer}
	\def\arraystretch{1.3}
	\begin{tabular}{l|c|c|}
		Reward Property & $\psi$ & $d_\psi$\\ \hline
        Bound range & $\psi = 0$ if $\left|r\right|  \leq R_{\text{max}}$ and $+\infty$ otherwise & $2 R_{\text{max}}\cdot\operatorname{TV}(\rho, \rho_{E})$ \\
		Smoothness & $\psi = 0$ if $\norm{r}_{\text{Lip}} \leq K$ and $+\infty$ otherwise & $K \cdot W_1(\rho, \rho_E)$ \\
		L2 Penalization & $\psi(r) = \alpha r^2$ & $\frac{1}{4 \alpha}\cdot  \chi^2(\rho, \rho_{E})$ \\ 
 \hline
	\end{tabular}
	\vskip-10pt
\end{table}

We find that these choices of regularizers\footnote{The additional scalar terms scale the entropy regularization strength and can be ignored in practice.} work very well in our experiments. In Appendix \ref{appx:B}, we further give a table for the well known  $f$-divergences, the corresponding $\phi$ and the learnt reward estimators, along with a result ablation on using different divergences. Compared to $\chi^2$, we find other $f$-divergences like  Jensen-Shannon result in similar performances but are not as readily interpretable.

\subsection{Inverse soft-Q update (Discrete control)}
Optimization along the optimal policy manifold gives the concave objective (Prop~\ref{prop:3}):
\begin{align}
\label{eq:method1}
    \underset{Q \in \Omega}{\max} \ \mathcal{J^*}(Q) = \mathbb{E}_{\rho_{E}}[\phi(Q(s, a) - \gamma \mathbb{E}_{s' \sim \mathcal{P}(\cdot|s,a)}V^*(s'))]  - (1- \gamma) \mathbb{E}_{\rho_0}[V^*(s_0)],
\end{align}
with $V^*(s) = \log \sum_{a} \exp {Q(s, a)}$.

For each $Q$, we get a corresponding reward $r(s, a)=Q(s, a)-\gamma \mathbb{E}_{s^{\prime} \sim \mathcal{P}(\cdot|s,a)} [\log \sum_{a^{\prime}} \exp {Q\left(s^{\prime}, a^{\prime}\right)}]
$. This correspondence is unique (Lemma \ref{lemma:Tstar} in Appendix), and every update step can be seen as finding a better reward for IRL. 

Note that estimating $V^*(s)$ exactly is only possible in discrete action spaces. Our objective forms a variant of soft-Q learning: to learn the optimal $Q$-function given an expert distribution.

\subsection{Inverse soft actor-critic update (Continuous control)}
\label{sec:Method1}
In continuous action spaces, it might not be possible to exactly obtain the optimal policy $\pi_Q$, which forms an energy-based model of the $Q$-function, and we use an explicit policy $\pi$ to approximate $\pi_Q$.

For any policy $\pi$, we have a objective (from Eq.~\ref{eq:psi_irl}):
\begin{equation}
 \label{eq:method2}
\mathcal{J}(\pi, Q) = \mathbb{E}_{\rho_{E}}[\phi(Q - \gamma \mathbb{E}_{s' \sim \mathcal{P}(\cdot|s,a)}V^\pi(s'))]  - (1- \gamma) \mathbb{E}_{\rho_0}[V^\pi(s_0)].
\end{equation}
For a fixed $Q$, soft actor-critic (SAC) update: $
\underset{\pi}{\max} \, \mathbb{E}_{s \sim \mathcal{D}, a \sim \pi(\cdot | s)} [Q(s, a) - \log \pi(a|s)]
$, brings $\pi$ closer to $\pi_Q$ while always minimizing Eq.~\ref{eq:method2} (Lemma \ref{lemma:sac} in Appendix).
Here $\mathcal{D}$ is the distribution of previously sampled states, or a replay buffer.

Thus, we obtain the modified actor-critic update rule to learn $Q$-functions from the expert distribution:
\begin{enumerate}
    \item For a fixed $\pi$, optimize $Q$ by maximizing $\mathcal{J}(\pi, Q)$.
    \item For a fixed $Q$, apply SAC update to optimize $\pi$ towards $\pi_Q$. 
\end{enumerate}

This differs from ValueDICE \cite{Kostrikov2020Imitation}, where the actor is updated adverserially and the objective may not always converge (Appendix \ref{Appx:C}).

\section{Practical Algorithm}

Algorithm~\ref{alg:ALG1} shows our $Q$-learning and actor-critic variants, with differences with conventional RL algorithms in red (we optimize -$\mathcal{J}$ to use gradient descent). We can implement our algorithm IQ-Learn in \textbf{15} lines of code on top of standard implementations of (soft) DQN~\cite{haarnoja2017reinforcement} for discrete control or soft actor-critic (SAC)~\cite{Haarnoja2018SoftAO} for continuous control, with a change on the objective for the $Q$-function. Default hyperparameters from \cite{haarnoja2017reinforcement, Haarnoja2018SoftAO} work well, except for tuning the entropy regularization. Target networks were helpful for continuous control. We elaborate details in Appendix~\ref{appx:D}.





\subsection{Training methodology}
\label{sec:training}
\looseness=-1
Corollary \ref{cor:xx} states $\mathbb{E}_{(s, a) \sim \mu}[V^\pi(s)-\gamma\mathbb{E}_{s' \sim \mathcal{P}(\cdot|s,a)} V^\pi(s')] = (1-\gamma) \mathbb{E}_{s \sim p_0}[V^\pi(s)]$, where $\mu$ is any policy's occupancy. We use this to stabilize training instead of using Eq.~\ref{eq:method1} directly.

\begin{wrapfigure}{R}{0.63\textwidth}
\vskip -15pt
\centering
\begin{minipage}[t]{.63\textwidth}
\begin{algorithm}[H]
    \small
    \caption{Inverse soft Q-Learning (both variants)}
    \label{alg:ALG1}
    \begin{algorithmic}[1]
    \State Initialize Q-function $Q_\theta$, and optionally a policy $\pi_\phi$
    \For {step $t$ in \{1...N\}}
        \State Train Q-function using objective from \autoref{eq:method1}:
        \newline \hspace*{1.25em}  $\theta_{t+1} \leftarrow \theta_{t} -\alpha_Q \nabla_{\theta} [{\color{red} \mathcal{-J}(\theta)}]$ 
        \newline \hspace*{1.25em} (Use $V^*$ for Q-learning and $V^{\pi_\phi}$ for actor-critic)
        \State (only with actor-critic) Improve policy $\pi_\phi$ with SAC style  \newline \hspace*{1.25em} actor update:
        \newline \hspace*{1.25em}  $\phi_{t+1} \leftarrow \phi_{t} + \alpha_\pi \nabla_{\phi} \mathbb{E}_{s \sim \mathcal{D}, a \sim \pi_\phi(\cdot | s)} [Q(s, a) - \log \pi_\phi(a|s)]$
    \EndFor
    \end{algorithmic}
\end{algorithm}
\end{minipage}

\begin{minipage}[t]{.63\textwidth}
\begin{algorithm}[H]
    \small
    \caption{Recover policy and reward}
    \label{alg:ALG2}
    \begin{algorithmic}[1]
    \State Given trained Q-function $Q_\theta$, and optionally a trained policy $\pi_\phi$
    \State Recover policy $\pi$:
    \newline \hspace*{1.25em} (Q-learning) $\pi \defeq \frac{1}{Z}{\exp Q_\theta}$ 
    \newline \hspace*{1.25em} (actor-critic) $\pi \defeq \pi_\phi$ 
    \State For state $\mathbf{s}$, action $\mathbf{a}$ and $\mathbf{s'} \sim \mathcal{P}(\cdot | \mathbf{s}, \mathbf{a})$
    \State Recover reward $r(\mathbf{s}, \mathbf{a}, \mathbf{s'})=Q_\theta(\mathbf{s}, \mathbf{a})-\gamma V^{\pi}\left(\mathbf{s'}\right)$
    \end{algorithmic}
\end{algorithm}
\end{minipage}
\vskip-20pt
\end{wrapfigure}

\looseness=-1
\textbf{Online}:
Instead of directly estimating $\mathbb{E}_{p_0}[V^\pi(s_0)]$ in our algorithm, we can sample $(s, a, s')$ from a replay buffer and get a single-sample estimate $\mathbb{E}_{(s, a, s') \sim \text{replay}}[V^\pi(s) - \gamma V^\pi(s')] $.
This removes the issue where we are only optimizing $Q$ in the inital states resulting in overfitting of $V^\pi(s_0)$, and improves the stability for convergence in our experiments. We find sampling half from the policy buffer and half from the expert distribution gives the best performances. 
Note that this is makes our learning online, requiring environment interactions.

\textbf{Offline}:
Although $\mathbb{E}_{p_0}[V^\pi(s_0)]$ can be estimated offline we still observe an overfitting issue. Instead of requiring policy samples we use only expert samples to estimate $\mathbb{E}_{(s, a, s') \sim \text{expert}}[V^\pi(s) - \gamma V^\pi(s')] $ to sufficiently approximate the term. This methodology gives us state-of-art results for offline IL.

\subsection{Recovering rewards}
Instead of the conventional reward function $r(s, a)$ on state and action pairs, our algorithm allows recovering rewards for each transition $(s, a, s')$ using the learnt $Q$-values as follows:
\begin{equation}
\label{eq:reward}
r(s, a, s^{\prime})=Q(s, a)-\gamma V^{\pi}\left(s^{\prime}\right)
\end{equation}
Now, $
\mathbb{E}_{s' \sim \mathcal{P}(\cdot|s,a)} [Q(s, a)-\gamma V^{\pi}\left(s^{\prime}\right)] = Q(s, a)-\gamma \mathbb{E}_{s' \sim \mathcal{P}(\cdot|s,a)} [ V^{\pi}\left(s^{\prime}\right)] = \mathcal{T}^\pi Q(s, a)$. This is just the reward function $r(s, a)$ we want. So by marginalizing over next-states, our expression correctly recovers the reward over state-actions. Thus, Eq.~\ref{eq:reward} gives the reward over transitions.

Our rewards require $s'$ which can be sampled from the environment, or by using a dynamics model.

\subsection{Implementation of Statistical Distances}
Implementing TV and $W_1$ distances is fairly trivial and we give details in Appendix \ref{appx:B}. For the \textbf{$\chi^2$-divergence}, we note that it corresponds to $\phi(x) = x - \frac{1}{4\alpha}x^2$. On substituting in Eq. \ref{eq:method1}, we get
\begin{equation*}
\label{eq:method_chi}
\resizebox{\textwidth}{!}{$
    \underset{Q \in \Omega}{\max} \ \mathbb{E}_{\rho_{E}}[(Q(s, a) - \gamma \mathbb{E}_{s' \sim \mathcal{P}(\cdot|s,a)}V^*(s'))]  - (1- \gamma) \mathbb{E}_{p_0}[V^*(s_0)] -\frac{1}{4\alpha} \mathbb{E}_{\rho_{E}}[(Q(s, a) - \gamma \mathbb{E}_{s' \sim \mathcal{P}(\cdot|s,a)}V^*(s'))^2]$}
\end{equation*}
In a fully offline setting, this can be further simplified as (using the offline methodology in Sec~\ref{sec:training}):
\begin{equation}
\label{eq:method_chi2}
    \underset{Q \in \Omega}{\min} \  -\mathbb{E}_{\rho_{E}}[(Q(s, a) - V^*(s))]  +\frac{1}{4\alpha} \mathbb{E}_{\rho_{E}}[(Q(s, a) - \gamma \mathbb{E}_{s' \sim \mathcal{P}(\cdot|s,a)}V^*(s'))^2]
\end{equation}

This is interestingly the same as the $Q$-learning objective in CQL \cite{kumar2020conservative}, a state-of-art method for offline RL  (using 0 rewards), and shares similarities with regularized behavior cloning~\cite{Reddy2020SQILIL}.\footnote{The simplification to get Eq.~\eqref{eq:method_chi2} is not applicable in the online IL setting where our method differs.}




\subsection{Learning state-only reward functions}

Previous works like AIRL \cite{fu2018learning} propose learning rewards that are only function of the state, and claim that these form of reward functions generalize between different MDPs. We find our method can predict state-only rewards by using the policy and expert state-marginals with a modification to Eq. \ref{eq:method1}:
\begin{align*}
\label{eq:state-rewards}
    \underset{Q \in \Omega}{\max} \ \mathcal{J^*}(Q) = \mathbb{E}_{s \sim \rho_{E}(s)}[\mathbb{E}_{a \sim \pi(\cdot|s)}[\phi(Q(s, a) - \gamma \mathbb{E}_{s' \sim \mathcal{P}(\cdot|s,a)}V^*(s'))]]  - (1- \gamma) \mathbb{E}_{p_0}[V^*(s_0)],
\end{align*}
with $\pi$ being here a stop gradient of $\pi_Q$. 
Interestingly, our objective no longer depends on the the expert actions $\pi_E$ and can be used for IL using only observations. For the sake of brevity, we expand on this in Appendix~\ref{sec:state_rewards}.




\section{Related Work}

\textbf{Classical IL}: 
Imitation learning has a long history, with early works using supervised learning to match a policy's actions to those of the expert \cite{hayes1994robot,sammut1992learning}.
A significant advance was made with the formulation of IL as the composition of RL and IRL~\cite{ng2000algorithms, abbeel2004apprenticeship, Ziebart2008MaximumEI},
recovering the expert's policy by inferring the expert's reward function, then finding the policy which maximizes reward under this reward function.
These early approaches required a hand-designed featurization of the MDP, limiting their applicability to complex MDPs. 
In this setting, early approaches ~\cite{dvijotham2010inverse, piot2016bridging} noted a formal equivalence between IRL and IL using an inverse Bellman operator similar to our own. 

\textbf{Online IL}: More recent work aims to leverage the power of modern machine learning approaches to learn good featurizations and extend IL to complex settings. Recent work generally falls into one of two settings: online or offline.
In the online setting, the IL algorithm is able to interact with the environment to obtain dynamics information. GAIL~\cite{Ho2016GenerativeAI} takes the nested RL/IRL formulation of earlier work 
, optimizing over all reward functions with a convex regularizer. This results in the objective in Eq.~\eqref{eq:irl}, with a max-min adversarial problem similar to a GAN~\cite{goodfellow2014generative}. A variety of further work has built on this adversarial approach~\cite{kostrikov2018discriminator,fu2018learning,baram2016model}. A separate line of work aims to simplify the problem in Eq.~\eqref{eq:irl} by using a fixed \(r\) or \(\pi\). In SQIL~\cite{Reddy2020SQILIL}, \(r\) is chosen to be the 1-0 indicator on the expert demonstrations, while ASAF~\cite{barde2020adversarial} takes the GAN approach and uses a discriminator (with role similar to \(r\)) of fixed form, consisting of a ratio of expert and learner densities. AdRIL~\cite{swamy2021moments} is a recent extension of SQIL, additionally assigning decaying negative reward to previous policy rollouts.

\textbf{Offline IL}: In the offline setting, the learner has no access to the environment. 
The simple behavioural cloning (BC)~\cite{ross2010efficient} approach is offline, but doesn't use any dynamics information. ValueDICE~\cite{Kostrikov2020Imitation} is a dynamics-aware offline approach with an objective somewhat similar to ours, motivated from minimization of a variational representation of the KL-divergence between expert and learner policies. ValueDICE requires adversarial optimization to learn the policy and Q-functions, with a biased gradient estimator for training. We show a way to recover a unbiased gradient estimate for the KL-divergence in Appendix
~\ref{Appx:C}.
The O-NAIL algorithm~\cite{arenz2020non} builds on ValueDICE and combines with a SAC update to obtain a method that is similar to our algorithm described in section \ref{sec:Method1}, with the specific choice of reverse KL-divergence as the relevant statistical distance.
The EDM method~\cite{jarrett2020strictly} incorporates dynamics via learning an explicit energy based model for the expert state occupancy, although some theoretical details have been called into question (see~\cite{swamyCritiqueStrictlyBatch2021} for details). The recent AVRIL approach~\cite{chan2021scalable} uses a variational method to solve a probabilistic formulation of IL, finding a posterior distribution over $r$ and $\pi$. Illustrating the potential benefits of alternative distances for IL, the PWIL~\cite{dadashi2021primal} algorithm gives a non-adversarial procedure to minimize the Wasserstein distance between expert and learned occupancies. The approach is specific to the primal form of the \({\cal W}_1\)-distance, while our method (when used with the Wasserstein distance) targets the dual form. 


\section{Experiments}

\subsection{Experimental Setup}
We compare IQ-Learn (``IQ'') to prior works on a diverse collection of RL tasks and environments - ranging from low-dimensional control tasks: CartPole, Acrobot, LunarLander - to more challenging continuous control MuJoCo tasks: HalfCheetah, Hopper, Walker and Ant. Furthermore, we test on the visually challenging Atari Suite  with high-dimensional image inputs. We compare on offline IL - with no access to the the environment while training, and online IL - with environment access. We show results on  $W_{1}$ and $\chi^{2}$ as our statistical distances, as we found them more effective than TV distance. In all cases, we train until convergence and average over multiple seeds. 
Hyperparameter settings and training details are detailed in Appendix~\ref{appx:D}.

\subsection{Benchmarks}

\paragraph{Offline IL} We compare to the state-of-art IL methods EDM and AVRIL, following the same experimental setting as \cite{chan2021scalable}. Furthermore, we compare with ValueDICE which also learns Q-functions, albeit with drawbacks such as adversarial optimization. We also experimented with SQIL, but found that it was not competitive in the offline setting. Finally, we utilize  BC as an additional IL baseline. 

\paragraph{Online IL}
We use MuJoCo and Atari environments and compare against state-of-art online IL methods: ValueDICE, SQIL and GAIL. We only show results on $\chi^2$ as $W_1$ was harder to stabilize on complex environments\footnote{$\chi^2$ and $W_1$ can be used together to still have a convex regularization and is more stable. 
}. 
Using target updates stabilizes the $Q$-learning on MuJoCo. 
For brevity, further online IL results are shown in Appendix~\ref{appx:D}.

\subsection{Results}

\begin{figure}[ht]
\vskip -5pt
\centering
\hspace*{-0.2cm}\includegraphics[width=1.02\linewidth]{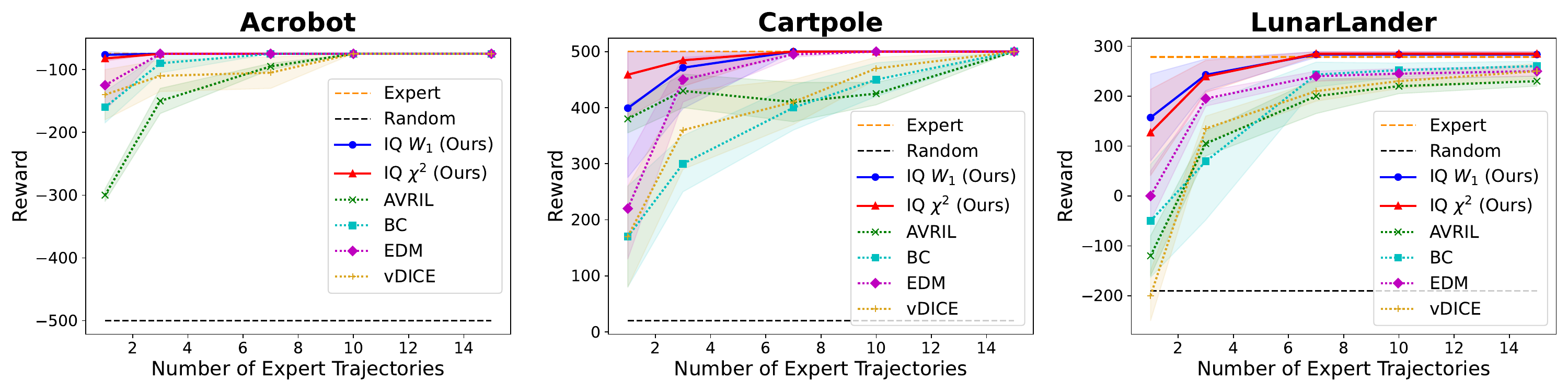}
\vskip -5pt
\caption{\small\textbf{Offline IL results.} We plot the average environment returns vs the number of expert trajectories.}
\label{fig:offline}
\vskip -10pt
\end{figure}

\paragraph{Offline IL} We present results on the three offline control tasks in Figure \ref{fig:offline}. On all tasks, IQ strongly outperforms prior works we compare to in performance and sample efficiency. Using just \textit{one expert trajectory}, we achieve expert performance on Acrobot and reach near expert on Cartpole.

\begin{wraptable}{R}{0.6\linewidth}
    \centering
	\small
	\vskip-13pt
	\tabcolsep 3pt
	\caption{\small \textbf{Mujoco Results.} We show our performance on MuJoCo control tasks using a single expert trajectory.}  
	\vskip-5pt
	\label{tbl:mujoco}
	\begin{tabular}{l|c|c|c|c||c}
		Task & GAIL & DAC & ValueDICE & IQ (Ours) & Expert \\ \hline
		Hopper & 3252.5 & 3305.1 & 3312.1 & \textbf{3546.4} & 3532.7  \\
	    Half-Cheetah & 3080.0 & 4080.6 & 3835.6  &  \textbf{5076.6} & 5098.3 \\
		Walker & 4013.7 & 4107.9 & 3842.6 & \textbf{5134.0} & 5274.5  \\ 
		Ant & 2299.1 & 1437.5 &  1806.3 & \textbf{4362.9} & 4700.0 \\
		Humanoid & 232.6 & 380.5 & 644.5 & \textbf{5227.1} & 5312.8 \\
		\hline
    \end{tabular}
    \vskip-15pt
\end{wraptable}

\paragraph{Mujoco Control}
We present our results on the MuJoCo tasks using a single expert demo in Table~\ref{tbl:mujoco}. IQ achieves expert-level performance in all the tasks while outperforming prior methods like ValueDICE and GAIL. We did not find SQIL competitive in this setting, and skip it for brevity.

\paragraph{Atari}
We present our results on Atari using 20 expert demos in Figure \ref{fig:atari}. We reach expert performance on Space Invaders while being near expert on Pong and Breakout. Compared to prior methods like SQIL, IQ obtains \textbf{3-7x} normalized score\footnote{Normalized rewards are obtained by setting random behavior to 0 and expert one to 1.} and converges in $\sim$300k steps, being \textbf{3x} faster compared to Q-learning based RL methods that take more than 1M steps to converge. Other popular methods like GAIL and ValueDICE perform near random even with 1M env steps.
 
\begin{figure}[h]
\vskip -8pt
\centering
\includegraphics[width=\linewidth]{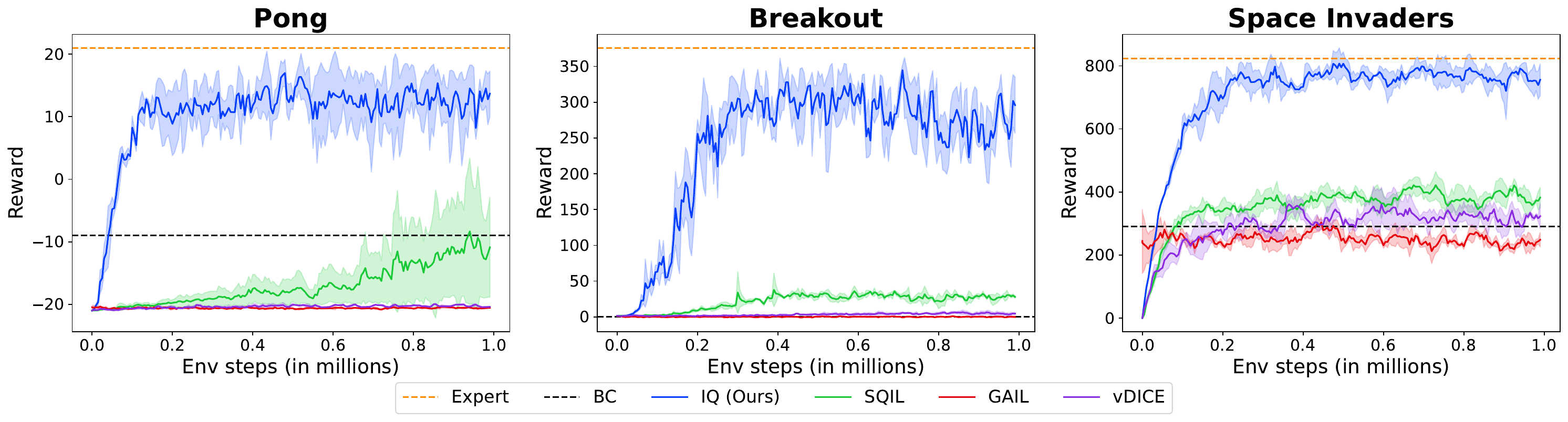}
\vskip -7pt
\caption{\small\textbf{Atari Results}. We show the returns vs the number of env steps (averaged over 5 seeds).}
\label{fig:atari}
\vskip -15pt
\end{figure}

\subsection{Recovered Rewards}
IQ has the added benefit of recovering rewards and can be used for IRL. On Hopper task, our learned rewards have a Pearson correlation of \textbf{0.99} with the true rewards. In Figure \ref{fig:rew_grid}, 
we visualize our recovered rewards in a simple grid environment. We elaborate details in Appendix \ref{appx:D}.


\begin{figure}[h!]
\vskip -10pt
\centering
\includegraphics[width=0.9\textwidth]{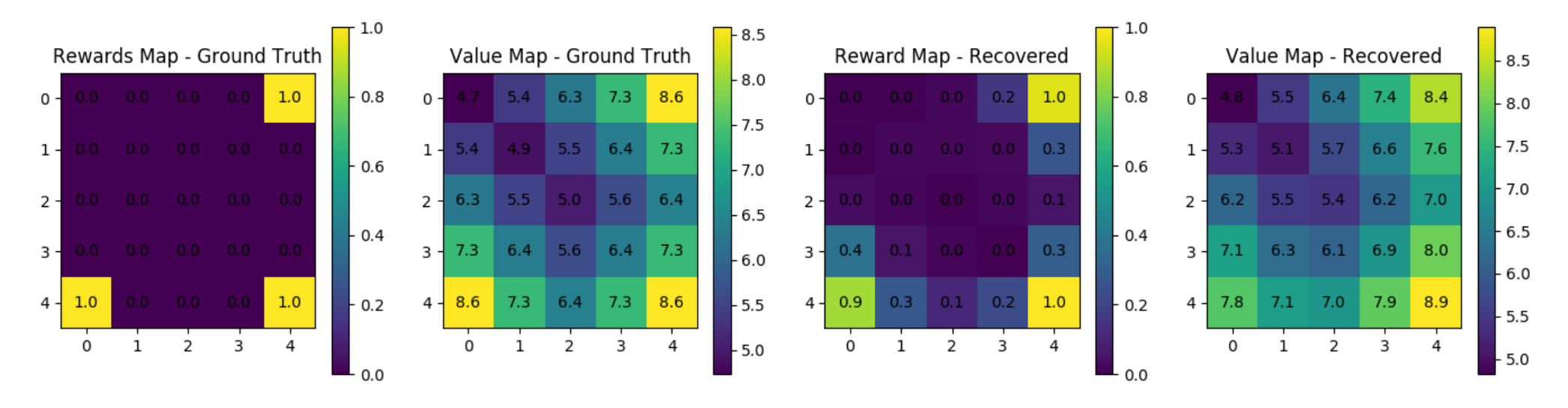}
\vskip -10pt
\caption{\small\textbf{Reward Visualization.} We use a discrete GridWorld environment with 5 possible actions: up, down, left, right, stay. Agent starts in a random state. (With 30 expert demos)}
\label{fig:rew_grid} 
\vskip -15pt
\end{figure}

\subsection{Robustness to Distribution Shifts}
We find IQ to be robust  to distribution shifts  between the expert and policy occupanices, and detail experiments with shift in the initial state distributions as well as goal distributions in Appendix~\ref{appx:F}. Overall we find that IQ shows good generalization performance to never seen before goals, and the capability to act as a meta-learner for IL.

\section{Discussion and Outlook}
We present a new principled framework for learning soft-$Q$ functions for IL and recovering the optimal policy and the reward, building on past works in IRL~\cite{Ziebart2008MaximumEI}. Our algorithm IQ-Learn outperforms prior methods with very sparse expert data and scales to complex image-based environments. We also recover rewards highly correlated with actual rewards. It has applications in autonomous driving and complex decision-making, but proper considerations need to be taken into account to ensure safety and reduce uncertainty, before any deployment. Finally, human or expert data can have errors that can propagate.
A limitation of our method is that our recovered rewards depend on the environment dynamics, preventing trivial use on reward transfer settings. One direction of future work could be to learn a reward model from the trained soft-$Q$ model to make the rewards explicit.

\section{Acknowledgements}

We thank Kuno Kim and John Schulman for helpful discussions. We also thank Ian Goodfellow as some initial motivations for this work were developed under an internship with him.

\section{Funding Transparency}

This research was supported in part by NSF (\#1651565, \#1522054, \#1733686), ONR (N00014-19-1-2145),
AFOSR (FA9550-19-1-0024) and FLI.
\bibliography{main.bib}
\bibliographystyle{plainnat}

\newpage

\title{Supplementary: Implicit IRL}

\appendix

\section{Appendix A}

\begin{figure}[H]
\vskip -15pt
\centering
\includegraphics[width=0.5\textwidth]{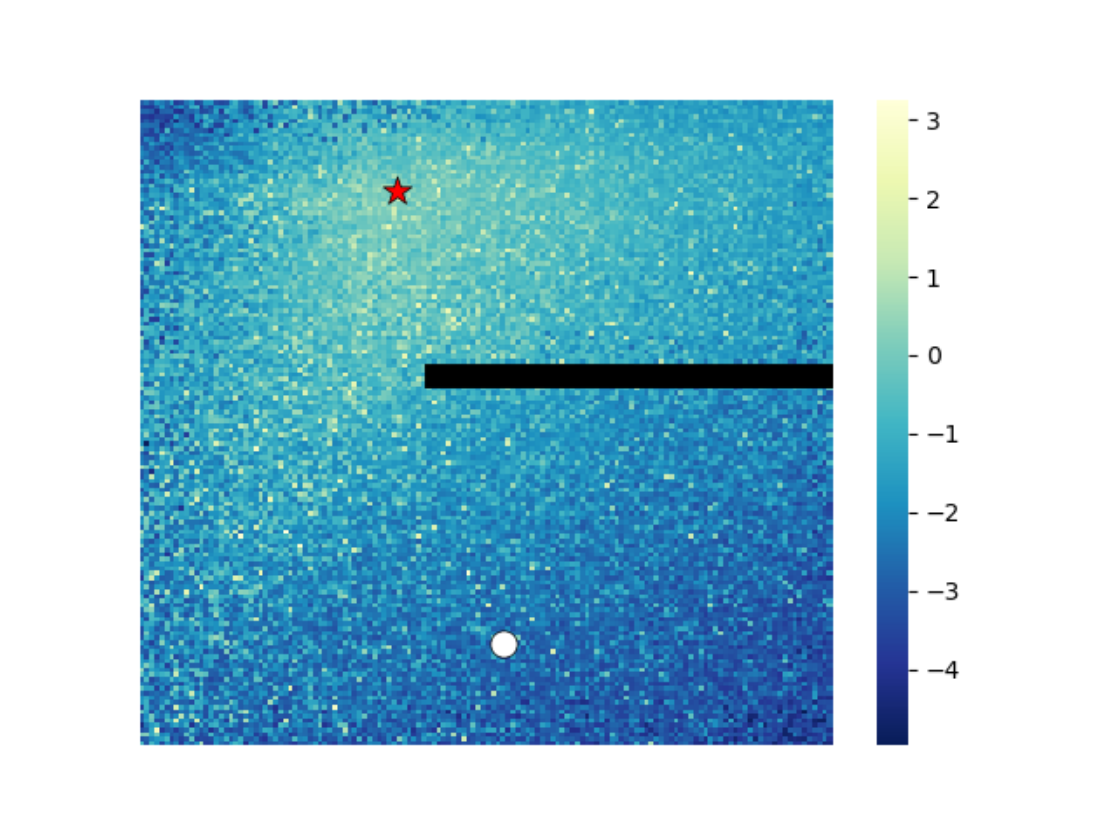}
\vskip -15pt
\caption{\small\textbf{State Rewards Visualization.} We visualize the state-only rewards recovered on a continuous control point maze task. The agent (white circle) has to reach the goal (red star) avoiding the barrier on right.}
\label{fig:pointmaze}
\vskip -10pt
\end{figure}

\subsection{Learning with state-only rewards}
\label{sec:state_rewards}
For a policy $\pi \in \Pi$, we define its state-marginal occupancy measure $\rho_{\pi}: \mathcal{S}  \rightarrow \mathbb{R}$ as $\rho_{\pi}(s)=(1-\gamma)\sum_{t=0}^{\infty} \gamma^{t} P\left(s_{t}=s | \pi\right)$. 

Suppose we are interested in learning rewards that are functions of only the states, then the Inverse-RL objective $L$ from Eq.~\ref{eq:irl} becomes a function of the state-marginal occupancies:
\begin{equation}
 \label{eq:state_irl}
 \underset{r \in \mathcal{R_\psi}}{\max} \min _{\pi \in \Pi} L_s(\pi, r) = \mathbb{E}_{s \sim \rho_{E}(s)}[\phi(r(s))]  - \mathbb{E}_{s \sim \rho(s)}[r(s)] - H(\pi). 
\end{equation}

Now, we can parameterize the rewards $r(s)$ using state-only value-functions $V(s)$ and remove the dependency on $Q(s, a)$. Then $V(s)$ can be learnt similar to learning $Q(s, a)$ in the main paper, but $Q(s,a)$ remains unknown and the optimal policy cannot be obtained simply as an energy-based model of $Q$.

Instead, we develop a new objective that can learn $Q$ while recovering state-only rewards below.

We expand the original objective $L$ using the expert occupancy:
\begin{align*}
      L(\pi, r) = \mathbb{E}_{s \sim \rho_{E}(s)} \mathbb{E}_{a \sim \pi_E(\cdot|s)} \left[\phi(r(s, a))  - \frac{\rho(s)\pi(a|s)}{\rho_E(s)\pi_E(a|s)} r(s, a)) \right] - H(\pi). 
\end{align*}

We see that the action dependency comes in the equation from the fact that we have ${\pi}/{\pi_E}$ inside.

Now, we propose to fix the expression to make it independent of actions by replacing the expert policy $\pi_E$ with the policy $\pi$. The new objective becomes:
\begin{align*}
    L'(\pi, r) = \mathbb{E}_{s \sim \rho_{E}(s)} \mathbb{E}_{a \sim \pi(\cdot|s)} \left[\phi(r(s, a))  - \frac{\rho(s)}{\rho_E(s)} r(s, a) \right] - H(\pi). 
\end{align*} 
Then for a fixed policy $\pi$, while maximizing over $r$ the constraint we have is that each reward component $r(s, a) \in R_\psi$. In a state $s$, $r(s, a)$ that maximizes the objective will take the same value independent of the action\footnote{The objective and the reward constraints remain same along each action dimension and a symmetry argument holds.}. Thus, the expectation over actions can be removed and this recovers Eq.~\ref{eq:state_irl}.


Writing the new objective using $Q$-functions, we get the modification to Eq. \ref{eq:method1}:
\begin{align}
\label{eq:state-rewards}
    \underset{Q \in \Omega}{\max} \ \mathcal{J^*}(Q) = \mathbb{E}_{s \sim \rho_{E}(s)}[\mathbb{E}_{a \sim \pi(\cdot|s)}[\phi(Q(s, a) - \gamma \mathbb{E}_{s' \sim P(s,a)}V^*(s'))]]  - (1- \gamma) \mathbb{E}_{p_0}[V^*(s_0)],
\end{align}
with $\pi$ set to $\stopgrad(\pi_Q)$ to prevent passing gradients through it.

This new objective does not depend on the the expert actions $\pi_E$ and can be used for IL using only observations (ILO). We visualize state-only rewards recovered on a 2D point mass navigation task in Fig~\ref{fig:pointmaze}. Notice that the rewards are not directional and are high on all sides of the target point, indicating they are not dependent on the action. 
We present additional results in Appendix~\ref{appx:D} and a theoretical guarantee in Appendix~\ref{subappx:monotonic}.



\subsection{Proofs for Section 3 and Section 4}
\label{subappx:proofs_1}



\paragraph{Proof for Lemma \ref{lemma:3.1}.}
Let ${P^\pi}$ be the (stochastic) transition matrix for the MDP corresponding to a policy $\pi$, such that for any $x \in \RSA$,  $P^\pi x(s,a) = \mathbb{E}_{s' \sim \mathcal{P}(\cdot|s,a), a' \sim \pi(\cdot |s') }\left[x(s', a')\right]$.

Let $r =\mathcal{T}^\pi Q$ for any $Q \in \RSA$. We expand  $\mathcal{T}^\pi$ in vector form over $\mathcal{S} \times \mathcal{A}$ using  ${P^\pi}$.
Then $\vr = \bm{Q} - \gamma
{P^\pi} (\bm{Q} - \log \bm{\pi})$.
 Here, $\left(I - \gamma {P^\pi}\right)$ is invertible as   $\norm{\gamma {P^\pi}} < 1$, for $\gamma < 1$, and the corresponding Neumann series converges. Thus $\bm{Q} = \left(I -\gamma {P^\pi}\right)^{-1} \left(\bm{r} - \log \bm{\pi}\right) + \log \bm{\pi}$. 
 So we see that for any $r \in \RSA$, there exists a unique preimage $Q \in \RSA$ proving that  $\mathcal{T}^\pi$ is a bijection.

Furthermore, on rearranging the vector form, we have $\bm{Q} = \vr  + \gamma
{P^\pi} (\bm{Q} - \log \bm{\pi})$.
This is just the vector expansion of the soft-bellmann operator $\mathcal{B}^\pi_r$, which has a unique contraction $Q$ for a given $r$. Thus, $Q =(\mathcal{T}^\pi)^{-1} r = \mathcal{B}^\pi_r Q$ for any $r \in \RSA$.



\begin{lemma}
\label{lemma:telescopic_basic}
Let the initial state distribution be $p_0(s)$, then for a policy $\pi$ and $V^\pi$ defined as before, we have $$\mathbb{E}_{(s, a) \sim \rho_\pi}[V^\pi(s) - \gamma\mathbb{E}_{s' \sim \mathcal{P}(\cdot|s,a)} V^\pi(s')] = (1-\gamma) \mathbb{E}_{s \sim p_0}[V^\pi(s)].$$
\end{lemma}
\begin{proof} We expand the discounted stationary distribution $\rho$ over state-actions and show the series forms a telescopic sum.
Let $p^\pi_t(s)$ be the marginal state distribution at time $t$ for a policy $\pi$.
Then,

$$
\begin{array}{l}
\mathbb{E}_{(s, a) \sim \rho_\pi}[V^\pi(s) - \gamma\mathbb{E}_{s' \sim \mathcal{P}(\cdot|s,a)} V^\pi(s')]  \\
=(1-\gamma) \sum_{t=0}^{\infty} \gamma^{t} \mathbb{E}_{s \sim p^\pi_{t}, a \sim \pi(s)}\left[V^\pi(s)-\gamma \mathbb{E}_{s^{\prime} \sim \mathcal{P}(\cdot|s,a)}V^\pi(s')\right] \\
=(1-\gamma) \sum_{t=0}^{\infty} \gamma^{t} \mathbb{E}_{s \sim p^\pi_{t}}[V^\pi(s)]-(1-\gamma) \sum_{t=0}^{\infty} \gamma^{t+1} \mathbb{E}_{s \sim p^\pi_{t+1}}[V^\pi(s)] \\
=(1-\gamma) \mathbb{E}_{s \sim p_0}[V^\pi(s)].
\end{array}
$$
\end{proof}

\begin{corollary}
\label{cor:xx}
In fact, for any valid occupancy measure $\mu$ over state-actions and $V^\pi$, it holds that $$\mathbb{E}_{(s, a) \sim \mu}[V^\pi(s) - \gamma\mathbb{E}_{s' \sim \mathcal{P}(\cdot|s,a)} V^\pi(s')] = (1-\gamma) \mathbb{E}_{s \sim p_0}[V^\pi(s)].$$
\end{corollary}

\begin{proof}
This relies on the fact that $V^\pi(s)$ is a function of only state and doesn't depend on the action. First, for any valid occupancy measure $\mu$, there exists a corresponding unique policy $\beta^\mu(a|s)$ s.t. $\beta^\mu$ generates $\mu$ \cite{Ho2016GenerativeAI}.

Let $p^\mu_t(s)$ be the marginal state distribution at timestep $t$ for the policy $\beta^\mu$.
Then,

$$
\begin{array}{l}
\mathbb{E}_{(s, a) \sim \mu}[V^\pi(s) - \gamma\mathbb{E}_{s' \sim \mathcal{P}(\cdot|s,a)} V^\pi(s')]  \\
=(1-\gamma) \sum_{t=0}^{\infty} \gamma^{t} \mathbb{E}_{s \sim p^{{\mu}}_{t}, a \sim \beta^\mu(s)}\left[V^\pi(s)-\gamma \mathbb{E}_{s^{\prime} \sim \mathcal{P}(\cdot|s,a)}V^\pi(s')\right] \\
=(1-\gamma) \sum_{t=0}^{\infty} \gamma^{t} \mathbb{E}_{s \sim p^\mu_{t}}[V^\pi(s)]-(1-\gamma) \sum_{t=0}^{\infty} \gamma^{t+1} \mathbb{E}_{s \sim p^\mu_{t+1}}[V^\pi(s')] \\
=(1-\gamma) \mathbb{E}_{s \sim p^\mu_0}[V^\pi(s)].
\end{array}
$$

Now $p^\mu_0$ is just the initial state distribution $p_0$ which is independent of the policy, thus giving our result.

\end{proof}

\begin{lemma}
\label{lemma:telescopic}
$\mathbb{E}_{\rho_\pi}[(\mathcal{T}^\pi Q)(s, a)] + H(\pi) = (1- \gamma) \mathbb{E}_{p_0}[V^\pi(s_0)]$, where $p_0(s)$ is the initial state distribution.
\end{lemma}

\begin{proof}
We can show this forms a telescopic series as in \cite{Nachum2019DualDICEEE} using lemma \ref{lemma:telescopic_basic} to depend only on the initial state distribution:
\begin{align*}
 \mathbb{E}_{\rho_\pi}[Q(s, a) - \gamma \mathbb{E}_{s' \sim \mathcal{P}(\cdot|s,a)}V^\pi(s')] + H(\pi) 
&=  \mathbb{E}_{\rho_\pi}[Q(s, a) - \gamma \mathbb{E}_{s' \sim \mathcal{P}(\cdot|s,a)}V^\pi(s') + H(\pi(a|s)] \\ 
&= \mathbb{E}_{\rho_\pi}[Q(s, a) -\log\pi(a|s) - \gamma \mathbb{E}_{s' \sim \mathcal{P}(\cdot|s,a)}V^\pi(s')] \\
&= \mathbb{E}_{\rho_\pi}[V^\pi(s) - \gamma\mathbb{E}_{s' \sim \mathcal{P}(\cdot|s,a)} V^\pi(s')] \\
&= (1- \gamma) \mathbb{E}_{p_0}[V^\pi(s)]. 
\end{align*}
This makes sense as the LHS and RHS both represent the max entropy RL objective, that is to maximize the cumulative sum of rewards or the expected value with respect to a policy for the initial state.
\end{proof}

\begin{lemma}
\label{lemma:sac}
SAC actor update decreases the objective $\mathcal{J}(\pi, Q)$ for the actor-critic update in main paper, wrt $\pi$ for a fixed $Q$.
\end{lemma}

\begin{proof}
$$
V^\pi(s) = \mathbb{E}_{a \sim \pi} [Q(s, a) - \log\pi (a|s)] = -D_{KL}\left(\pi(\cdot | s) \lVert \frac{1}{Z_s}\exp(Q(s, \cdot)\right) +  \log(Z_s),
$$
where $Z_s$ is the normalizing factor
$\sum_{a} \exp{Q(s, a)}$

Now, for a policy $\pi'$ the the SAC actor update rule \cite{Haarnoja2018SoftAO} is $
\underset{\pi'}{\arg\min} \, D_{KL}\left(\pi' \lVert \frac{1}{Z}\exp(Q)\right)
$.

Thus, if $\pi$ is the policy obtained on applying the SAC actor update to $\pi'$, we have $V^{\pi}(s) > V^{\pi'}(s)$. So, as long as $\phi$ in $\mathcal{J}$ is a monotonically non-decreasing function, this implies $\mathcal{J}(\pi, Q) < \mathcal{J}(\pi', Q)$.
\end{proof}

\section{Appendix B}
\label{appx:B}







\paragraph{Integral Probability Metric (IPM)}
An IPM parameterized by $\mathcal{F}$ between two distributions $P$ and $Q$ is defined as 
\begin{align}
    \gamma_{\mathcal{F}}(P, Q):=\sup _{f \in \mathcal{F}}\left|\mathbb{E}_{P} f(X) -\mathbb{E}_{Q} f(X) \right|.
\end{align}

Suppose $\mathcal{F}$ is such that $f \in \mathcal{F} \Rightarrow-f \in \mathcal{F}$. Then,
\begin{align}
\label{eq:ipm}
\gamma_{\mathcal{F}}(P, Q)=\sup _{f \in \mathcal{F}}|\mathbb{E}_{P} f-\mathbb{E}_{Q} f|=\sup _{f \in \mathcal{F}} \ \mathbb{E}_{P} f-\mathbb{E}_{Q} f.
\end{align}

Some IPMs that satisfy this symmetry are: Dudley metric, Wasserstein metric, total variation distance, Maximum Mean Discrepancy (MMD). 

We can see that for $\phi = \mathcal{I}, R_\psi = \mathcal{F}$, Eq.~\ref{eq:gen_dist} reduces to Eq.~\ref{eq:ipm}. 

\paragraph{$f$-divergence}
The $f$-divergence between two distributions ${P}$ and ${Q}$ is defined using the convex conjugate $f^*$ as
\begin{align}
D_{f}(P \| Q)=\mathbb{E}_{Q}\left[f\left(\frac{P}{Q}\right)\right]=\sup _{g: \mathcal{X} \rightarrow \mathbb{R}} \mathbb{E}_{P}[g(X)]-\mathbb{E}_{Q}\left[f^{*}(g(X))\right].
\end{align}

Interpreting $g = -r$,
\begin{align}
\label{eq:f-dist}
D_{f}(P \| Q) &= \sup _{r: \mathcal{X} \rightarrow \mathbb{R}} \mathbb{E}_{P}[-r(X)]-\mathbb{E}_{Q}\left[f^{*}(-r(X))\right] \\
&= \sup _{r: \mathcal{X} \rightarrow \mathbb{R}} \mathbb{E}_{Q}\left[-f^{*}(-r)\right]  - \mathbb{E}_{P}[r].
\end{align}

Thus, for $\phi(x) = -f^{*}(-x), R_\psi = \RSA$, Eq.~\ref{eq:gen_dist} reduces to Eq.~\ref{eq:f-dist}.

\begin{table}[h]
\small
	\centering
	\caption{\small {List of divergence functions, convex conjugates, $\phi$ and optimal reward estimators}}   %
	\label{tbl:div}
	\vskip5pt
	\def\arraystretch{1.5}
	\begin{tabular}{l|c|c|c|c}
        
		Divergence & $f(t)$ & $f^*(u)$ & $\phi(x)$ & $r$ \\ \hline
        Forward KL & $-\log{t}$ & $-1 -\log(-u)$ & $1+\log{x}$ & $\frac{\rho_E}{\rho}$\\
		Reverse KL & $t\log{t}$ & $e^{(u-1)}$ & $-e^{-(x+1)}$ & $-(1+ \log \frac{\rho}{\rho_{E}})$ \\
		Squared Hellinger & $(\sqrt{t} - 1)^2$ & $\frac{u}{1-u}$ & $\frac{x}{1+x}$ & $\sqrt{\frac{\rho_E}{\rho}}-1$\\ 
		Pearson $\chi ^{2}$ & $(t-1)^2$ & $u + \frac{u^2}{4}$ & $x - \frac{x^2}{4}$  & $2(1-\frac{\rho}{\rho_{E}})$\\
		Total variation & $\frac{1}{2} | t-1| $ & $u$ & $x$ & $\frac{1}{2} \sign{(1-\frac{\rho}{\rho_{E}})}$ \\
		Jensen-Shannon & $-(t+1)\log(\frac{t+1}{2}) + t \log{t}$ & $-\log{(2- e^u)}$ & $\log{(2- e^{-x})}$ & $\log{\frac{1}{2}(1+\frac{\rho_E}{\rho})}$ \\
 \hline
	\end{tabular}
	\vskip-10pt
\end{table}

\subsection{Implementation of Statistical Distances}

\paragraph{Total Variation}
Total variation gives a constraint on reward functions: $\left|r\right|  \leq \frac{1}{2} $.

As $Q_{t'} = \sum_{t=t'}^{\infty} \gamma^t r(s_t, a_t) + \gamma^t H(a_t|s_t)$, we obtain a constraint on $Q$:

$\left|Q\right| \leq \frac{1}{1-\gamma} (R_{max} + \log|A|) = \frac{1}{1-\gamma} (\frac{1}{2} + \log|A|)
$

This can be easily enforced by bounding $Q$ to this range using a $tanh$ activation.

\paragraph{$W_1$ Distance}
For Wasserstein-1 distance, we use gradient penalty~\cite{Gulrajani2017ImprovedTO} to enforce the Lipschitz constraint, although other techniques like spectral normalization~\cite{Miyato2018SpectralNF} can also be utilized.

\paragraph{$\chi^2$-divergence}
$\chi^2$-divergence corresponds to an $f$-divergence with a choice of $f(x) = (x-1)^2$.

We generalize this to a choice of $f(x) = \alpha (x-1)^2$ with $\alpha > 0$, which scales the original divergence by a constant factor of $\alpha$.  

Then $\phi(x) = -f^*(-x) = x - \frac{1}{4\alpha} x^2$. It corresponds to using a (strong) convex reward regularizer $\psi(r) = \frac{1}{4\alpha}r^2$.

\subsection{Effect of different Divergences}

\begin{figure}[ht]
\vskip -10pt
\centering
\includegraphics[width=\textwidth]{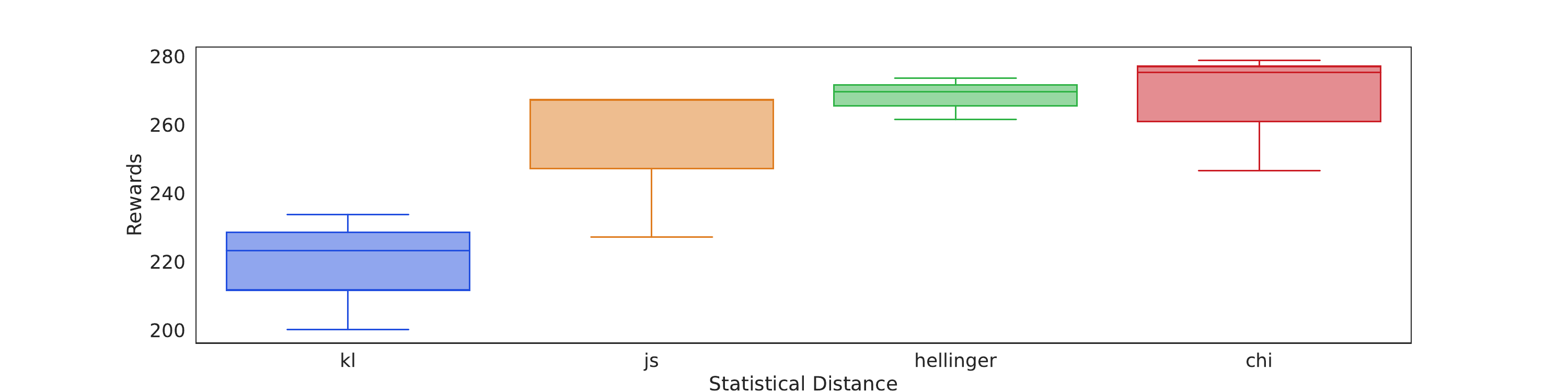}
\vskip -5pt
\caption{\small\textbf{Divergence ablation.} We show environment returns for different divergences on LunarLander.}
\label{fig:div} 
\vskip -5pt
\end{figure}

We test IQ-Learn with different divergences:  Jensen-Shannon (JS), Hellinger, KL and $\chi^2$ divergence. We use the LunarLander environment with our offline IL experimental settings and a single expert trajectory. All experiments are repeated over 10 seeds. We show a box-plot of the environment returns for different divergences and find that JS, Hellinger and $\chi^2$ divergence perform similarly, consistent with the findings on different type of GANs~\cite{Lucic2018AreGC}. Here, KL-divergence performs worse and is suboptimal compared to the other divergences.

\section{Appendix C}
\label{Appx:C}

In this section, we expand over our analysis in Section 3 and present proof of properties over the $Q$-policy space: Propositions \ref{prop:1}, \ref{prop:2}, \ref{prop:3} in main paper.

\looseness=-1
For simplicity, we define a concave function $\phi: \mathbb{R} \rightarrow \mathbb{R} \cup \{-\infty\}$ such that $g$ is given as $g(x) \defeq x - \phi(x)$, same as in Section 4 of the main paper. We are interested in regularizers $\psi$ induced by $g$, such that 
\begin{align}
\psi_{g}(r)=\mathbb{E}_{\rho_{E}} [{g}(r(s, a))].
\end{align}

We simplify the IRL objective (from Eq.~\ref{eq:psi_irl}):
\begin{align*}
\mathcal{J}(\pi, Q) &= \mathbb{E}_{ \rho_{E}}[\mathcal{T}^\pi Q] 
- (1- \gamma) \mathbb{E}_{ p_0}[V^\pi(s_0)] -  \psi(\mathcal{T}^\pi Q) \\
&= \mathbb{E}_{\rho_{E}}[\mathcal{T}^\pi Q] - (1- \gamma) \mathbb{E}_{\rho_0}[V^\pi(s_0)] - \E_{\rho_E}[\mathcal{T}^\pi Q - \phi(\mathcal{T}^\pi Q)] \\
&= \mathbb{E}_{\rho_{E}}[\phi(Q - \gamma \mathbb{E}_{s' \sim \mathcal{P}(\cdot|s,a)}V^\pi(s'))] - (1- \gamma) \mathbb{E}_{\rho_0}[V^\pi(s_0)].
\end{align*}


\begin{lemma}
$\mathcal{J}(\pi, \cdot)$ is concave for all $\pi \in \Pi$.
\end{lemma}

\begin{proof}
Let $Q_1, Q_2 \in \Omega$ and suppose $\lambda \in [0, 1]$. We rely on the fact that the regularized IRL objective $L(\pi, \cdot)$ is concave for all $\pi$. Note that $r=\mathcal{T}^\pi Q$ is an affine transform of $Q$, given in vector form as $\bm{r} - \log \bm{\pi}= \left(I -P^\pi\right) \bm{Q}$. Thus, $\mathcal{T}^\pi(\lambda Q_1 + (1 - \lambda) Q_2) = \lambda \mathcal{T}^\pi Q_1 + (1 - \lambda) \mathcal{T}^\pi Q_2$.
\begin{align*}
 \mathcal{J}(\pi, \lambda Q_1 + (1-\lambda) Q_2 )
&=  L(\pi, \mathcal{T}^\pi (\lambda Q_1 + (1-\lambda) Q_2)) \\ 
&= L(\pi, \lambda \mathcal{T}^\pi Q_1 + (1 - \lambda) \mathcal{T}^\pi Q_2) \\ 
&\geq \lambda L(\pi,  \mathcal{T}^\pi Q_1) + (1 - \lambda) L(\pi,  \mathcal{T}^\pi Q_2)\\
&= \lambda \mathcal{J}(\pi, Q_1) +  (1 -\lambda) \mathcal{J}(\pi, Q_2).  
\end{align*}
Thus, $\mathcal{J}(\pi, \cdot)$ is concave.
\end{proof}

For building up our analysis, we will first prove the saddle point properties of $\J(\pi, Q)$ by adding a monotonicity assumption on  $\phi$ that it is a non-decreasing function. We then generalize the proof to show that these properties hold for any concave $\phi$ in Section~\ref{sec:general}.

\begin{lemma}
\label{lemma:C2}
For $\psi_g$ corresponding to a non-decreasing $\phi, \mathcal{J}(\cdot, Q)$ 
has a unique minima  $\pi_Q = \frac{1}{Z_s}\exp(Q)$ with normalizing factor  $Z_s = \sum_{a} \exp{Q(s, a)}$.
\end{lemma}

\begin{proof}
We have,
$$
V^\pi(s) = \mathbb{E}_{a \sim \pi} [Q(s, a) - \log\pi (a|s)] = -D_{KL}\left(\pi(\cdot | s) \lVert \frac{1}{Z_s}\exp(Q(s, \cdot)\right) +  \log(Z_s).
$$
For a fixed $Q$, the KL divergence is strictly convex in $\pi$ with minima at $\pi_Q$, implying $V^\pi(s)$ is strictly concave in $\pi$ . Similarly, $r=\mathcal{T}^\pi Q = Q - \gamma \mathbb{E}_{s' \sim \mathcal{P}(\cdot|s,a)}V^\pi(s')$ is strictly convex in $\pi$ with minima at $\pi_Q$. Now, as $\phi$ is a  non-decreasing function, $\ \mathbb{E}_{\rho_E}[\phi(\mathcal{T}^\pi Q)]$ will be minimum at $\pi_Q$ and will be always non-decreasing as we pull away. Similarly the second term of $\J$, given as $-(1-\gamma) \E_{\rho_0}[V^\pi(s_0)]$ is convex with a minima at $\pi_Q$. Thus $\mathcal{J}(\pi, Q) > \mathcal{J}(\pi_Q, Q)$,  for any $\pi \neq \pi_Q$. This is sufficient to establish that
$J(\cdot, Q)$ has a unique minima at $\pi_Q$.
\end{proof}

\begin{lemma}
\label{lemma:Tstar}
Define $\mathcal{T}^*: \mathbb{R^{\mathcal{S} \times \mathcal{A}}} \rightarrow  \mathbb{R^{\mathcal{S} \times \mathcal{A}}}$ such that 
\begin{align*}
\left(\mathcal{T}^{*} Q\right)(s, a)=Q(s, a)-\gamma \mathbb{E}_{s^{\prime} \sim \mathcal{P}(\cdot|s,a)} [\log \sum_{a^{\prime}} \exp{Q\left(s^{\prime}, a^{\prime}\right)}].
\end{align*}
Then $\mathcal{T}^*$ is bijective.
\end{lemma}

\begin{proof}
For $r=\mathcal{T}^* Q$, we have $Q(s, a)=r(s, a) + \gamma \mathbb{E}_{s^{\prime} \sim \mathcal{P}(\cdot|s,a)} [\log \sum_{a^{\prime}} \exp{Q\left(s^{\prime}, a^{\prime}\right)}]$. 
This is just the soft Bellman equation  (Eq.~\ref{eq:q}), for which a unique contraction $Q^*$ exists satisfying it \cite{haarnoja2017reinforcement}. Thus for any $r$, we have a unique preimage
$Q^*$ such that $r=\mathcal{T}^* Q^*$.
Hence,  $\mathcal{T}^*$ is a bijection.
\end{proof}

\begin{lemma}
    \label{lemma:pi_q_star}
We have that $\T^* Q = \T^{\pi_Q} Q$. Moreover, for $Q\in\R^{\states\times\actions}$, let $r = \T^* Q$. Then, the optimal (soft) policy with respect to $r$ satisfies $\pi^*_r = \pi_Q$. This notably implies that
    \begin{equation}
        \pi_Q = \argmax_{\pi} \E_{s,a\sim\rho_\pi}[(\T^* Q)(s,a) - \ln \pi(a|s)].
    \end{equation}
\end{lemma}
\begin{proof}
    The first holds is true by basic properties of the Legendre-Fenchel transform~\cite[Appx.~A]{vieillard2020leverage}.  Here, $Q$ is the fixed point of the optimal Bellman operator $\mathcal{B}^*$ for reward $r$, so $\pi_Q$ is the optimal policy.
\end{proof}

\begin{lemma}
\label{lemma:C3}
We have that a unique saddle point exists for $\mathcal{J}(\pi, Q)$ implying $\underset{\pi \in \Pi}{\min} \ \underset{Q \in \Omega}{\max} \ \mathcal{J}(\pi, Q) = \underset{Q \in \Omega}{\max}  \ \underset{\pi \in \Pi}{\min}  \ \mathcal{J}(\pi, Q)$.
\end{lemma}


Let $(\pi^*, r^*)$ be the unique saddle point for $L$. We will first solve for the min-max of $\J$.

As $r = \mathcal{T}^\pi Q$ is an affine transform of $Q$ for a fixed $\pi$, we have
$$
\pi^* = \argmin_{\pi \in \Pi} \underset{r \in \mathcal{R}}{\max} \ L(\pi, r) = 
 \argmin_{\pi \in \Pi} \underset{Q \in \Omega}{\max} \ L(\pi, \mathcal{T}^\pi Q) =  \argmin_{\pi \in \Pi} \underset{Q \in \Omega}{\max} \ \mathcal{J}(\pi, Q)
$$

Thus, $\argmin_\pi \max_Q \mathcal{J}$ coincides with the first coordinate of the saddle point for $L$. Now, we can relate the second coordinates.


For $r^* = \argmax_r L(\pi^*, r)$, as $L$ satisfies the minimax theorem, we necessarily have that 
\begin{equation*}
    \pi^* = \argmin_\pi L(\pi, r^*).
\end{equation*}
So, as $L(\pi, r^*) = \E_{s,a\sim\rho_E}[\phi(r^*(s,a))] - \E_{s,a\sim\rho_\pi}[r^*(s,a)-\ln\pi(a|s)]$, this means that $\pi^*$ is the optimal policy for $r^*$. Write $Q^* = (\mathcal{T}^*)^{-1}r^*$ the associated optimal $Q$-function, we have that $\pi^* = \pi_{Q^*}$. So, using the affine transformation property, we have
\begin{align*}
    r^* &= \argmax_r L(\pi^*, r)
    \\
    &= \mathcal{T}^{\pi^*}\left(\argmax_Q L(\pi^*, \mathcal{T}^{\pi^*} Q)\right)
    \\
    &= \mathcal{T}^{\pi_{Q^*}}\left(\argmax_Q L(\pi_{Q^*}, \mathcal{T}^{\pi_{Q^*}} Q)\right)
    \\
    &= \mathcal{T}^{\pi_{Q^*}}\left(Q^\dagger\right).
\end{align*}
where $Q^\dagger := \argmax_Q L(\pi_{Q^*}, \mathcal{T}^{\pi_{Q^*}} Q)$.

As $r^* = \mathcal{T}^* Q^*$ by definition of $Q^*$, we have $\mathcal{T}^* Q^* = \mathcal{T}^{\pi_{Q^*}}Q^\dagger$. We also know that $\mathcal{T}^* Q = \mathcal{T}^{\pi_Q} Q$, so $\mathcal{T}^* Q^* = \mathcal{T}^{\pi_{Q^*}} Q^* = \mathcal{T}^{\pi_{Q^*}}Q^\dagger$. Composing with $(\mathcal{T}^{\pi_{Q^*}})^{-1}$ we obtain $ Q^\dagger=Q^*$.

Overall, with $(\pi^*,r^*)$ the unique saddle point of $L$, having defined $Q^* = (\mathcal{T}^*)^{-1} r^*$, we have shown
\begin{equation*}
    \min_{\pi\in\Pi} \max_{Q\in\Omega} \mathcal{J}(\pi, Q) = \mathcal{J}(\pi^*, Q^*),
\end{equation*}
and $\pi^* = \pi_{Q^*}$.

Now, we show the same holds for the max-min of $\J$. We can relate $Q^{**} := \argmax_Q \min_\pi \J(\pi,Q)$ and $\pi_{Q^{**}}$ to the saddle point of $L(\pi, r)$. 
We have 
\begin{align*}
    \min_\pi \J(\pi, Q) &= \min_\pi L(\pi, \T^\pi Q) &\text{ by def. of $\J$}
    \\
    &= L(\pi_Q, \T^{\pi_Q} Q) &\text{ by Lemma~\ref{lemma:C2}}
    \\
    &= L(\pi_Q, \T^* Q) &\text{ by Lemma~\ref{lemma:pi_q_star}}
    \\
    &= \E_{\rho_E}[\phi(\T^* Q)] - \E_{\rho_{\pi_Q}}[(\T^* Q)- \ln\pi_Q] &\text{ by def. of $L$}
    \\
    &= \min_\pi \E_{\rho_E}[\phi(\T^* Q)] - \E_{\rho_{\pi}}[(\T^* Q)- \ln\pi]  &\text{ by Lemma~\ref{lemma:pi_q_star}}
    \\
    &= \min_\pi L(\pi, \T^*Q).
\end{align*}
We therefore have that,
$$
    Q^{**} 
    = \argmax_Q \min_\pi \J(\pi, Q)
    = \argmax_Q \min_\pi L(\pi, \T^*Q).
$$
Recalling that $(\pi^*, r^*)$ is the saddle point of $L$, we have 
$r^* = \T^* Q^{**}$  as  $\T^*$ is bijective. However, by definition $Q^* = (\mathcal{T}^*)^{-1} r^*$, which readily implies that $Q^{**} = Q^*$.

We have just shown that
\begin{equation*}
    \max_{Q\in\Omega}\min_{\pi\in\Pi}  \mathcal{J}(\pi, Q) = \mathcal{J}(\pi_{Q^*}, Q^*) = \mathcal{J}(\pi^*,Q^*) = \min_{\pi\in\Pi}\max_{Q\in\Omega}  \mathcal{J}(\pi, Q).
\end{equation*}

Therefore, the saddle point of $L$ uniquely corresponds to the saddle point $(\pi^*, Q^*)$ of $\mathcal{J}$, given as $(\pi^*,r^*)$ for $r^*= \mathcal{T}^{\pi^*} Q^*$.

This forms the proof for Proposition \ref{prop:1}, \ref{prop:2}.

\paragraph{Proof for Proposition \ref{prop:3}}
We have,
$$
    \mathcal{J^*}(Q) = \mathbb{E}_{\rho_{E}}[\phi(Q(s, a) - \gamma \mathbb{E}_{s' \sim \mathcal{P}(\cdot|s,a)}V^*(s'))]  - (1- \gamma) \mathbb{E}_{p_0}[V^*(s_0)].
$$

As log-sum-exp is convex, $V^*(s) = \log \sum_{a} \exp {Q(s, a)}$ is convex in Q. Then concavity follows from the fact that the first term, $\phi(Q(s, a) - \gamma \mathbb{E}_{s' \sim \mathcal{P}(\cdot|s,a)}V^*(s'))$ is concave, as it is a concave function composed with a non-decreasing concave function. 

\subsection{Generalization}
\label{sec:general}

\begin{figure}[H]
\centering
\includegraphics[width=0.8\textwidth, scale=0.2]{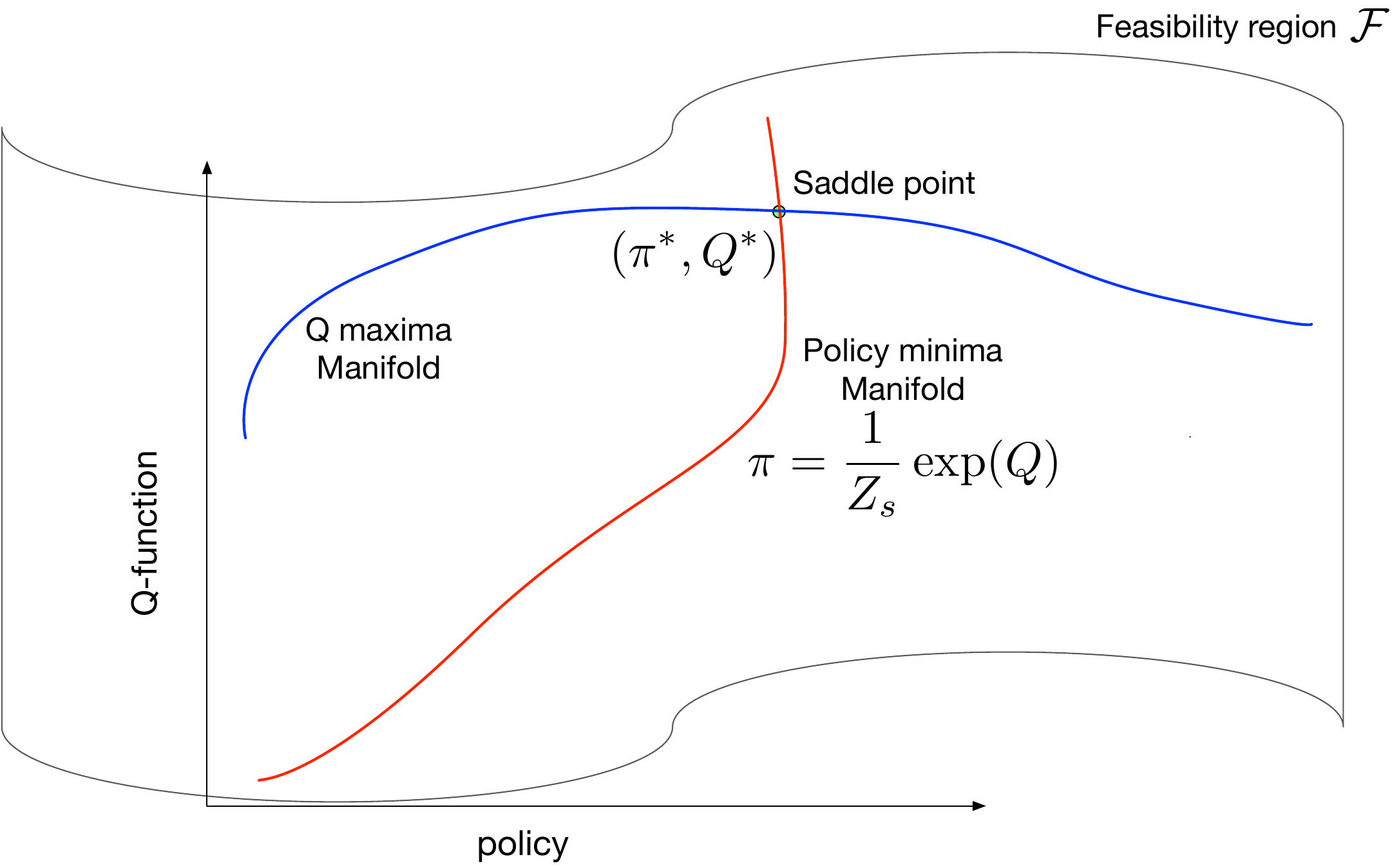}
\vskip -5pt
\caption{\small\textbf{Feasibility region in Q-policy space.}}
\label{fig:feasible}
\vskip -10pt
\end{figure}

In the above section, we made a monotonicity assumption on $\phi$ in Lemma C.2. We show that we can relax this assumption and the saddle point properties still hold, although $\mathcal{J}$ is not so well-behaved everywhere anymore.

For a fixed $\pi$, the optimizer of the concave problem, $\max_r L(\pi, r) = \mathbb{E}_{\rho_{E}}[\phi(r(s, a))]  - \mathbb{E}_{\rho}[r(s, a)] - H(\pi) $ satisfies\footnote{A concave function may not be differentiable everywhere and in general, we get a condition on the subdifferential of $\phi$: $\rho/\rho_E \in \partial \phi(r)$.}:

$$
\phi'(r) \rho_E  - \rho  = 0.
$$

Thus, $\phi'(r(s, a)) = {\rho(s, a)}/{\rho_E(s, a)} \in [0, \infty)$. This tells us that there exists a set of rewards $\mathcal{R}_\phi$, such that $\phi$ is non-decreasing on this set. For a concave $\phi$, $\mathcal{R}_\phi$ is just the convex set of reals that are on the left of its maxima.

\begin{lemma}
\label{lemma:C3}
Define a convex \textbf{feasibility region} on the Q-policy space:
$$
\mathcal{F}_\phi = \{(\pi, Q): r = \mathcal{T}^\pi Q \in \mathcal{S} \times \mathcal{A} \rightarrow \mathcal{R_\phi} \}.
$$
Then, for a given $\pi$, any optimal $Q= \argmax_{Q'}  \mathcal{J}(\pi, Q')$  has to lie in $\mathcal{F}_\phi$.
\end{lemma}

\begin{proof}
If $Q$ is optimal, then $\mathcal{T}^\pi Q$ maximizes $L(\cdot, r)$, and so it's corresponding $r$ is optimal. For a fixed $\pi$, and any $(s, a) \in \mathcal{S} \times \mathcal{A}$, the optimal reward has to satisfy $\phi'(r(s, a)) = {\rho(s, a)}/{\rho_E(s, a)}$. Thus, each component of the reward vector lies in $\mathcal{R_\phi}$. This tells us $\mathcal{T}^\pi Q$ lies in the required region.
\end{proof}

We get two properties in the feasibility region $\mathcal{F}_\phi$:
\begin{enumerate}
\item $\argmax_Q \mathcal{J}(\cdot, Q)$ lies in $\mathcal{F}_\phi$,
\item $\phi$ is non-decreasing, so lemma C.2 holds in this region.
\end{enumerate}

We just need one last lemma to prove the existence of a unique saddle point:

\begin{lemma}
\label{lemma:C4}
A saddle point exists only at the intersection of two curves: $\argmax_Q \mathcal{J}(\cdot, Q)$ and  $\argmin_\pi \mathcal{J}(\pi, \cdot)$.
\end{lemma}

\begin{proof}
 We parameterize the curves $f(\pi) = \argmax_Q \mathcal{J}(\pi, Q)$ and  $g(Q) = \argmin_\pi \mathcal{J}(\pi, Q)$. A saddle point has to satisfy $ \underset{\pi}{\min} \ \underset{Q}{\max} \ \mathcal{J}(\pi, Q) = \underset{Q}{\max} \ \underset{\pi}{\min}   \ \mathcal{J}(\pi, Q)$. This implies, $ \underset{\pi}{\min} \ \mathcal{J}(\pi, f(\pi)) = \underset{Q}{\max}  \ \mathcal{J}(g(Q), Q)$. This equation can only be satisfied when both the curves intersect.
 
 Therefore, any saddle point lies at the intersection of the Q-maxima and policy minima curves.
\end{proof}

We have established that within the feasibility region $\mathcal{F}_\phi$, lemma C.1 and C.2 hold. Thus, there exists a single saddle point in this region. Furthermore, $\argmax_Q \mathcal{J}(\cdot, Q)$ lies in $\mathcal{F}_\phi$ so lemma C.4 tells us there cannot exist any other saddle points outside $\mathcal{F}_\phi$.

This completes our proof of the existence of a unique saddle point of $\mathcal{J}$ for any concave $\phi$.

We summarize these properties in Fig~\ref{fig:feasible}.
\subsection{Convergence Guarantee}

\begin{figure}[ht]
\centering
\includegraphics[width=\textwidth]{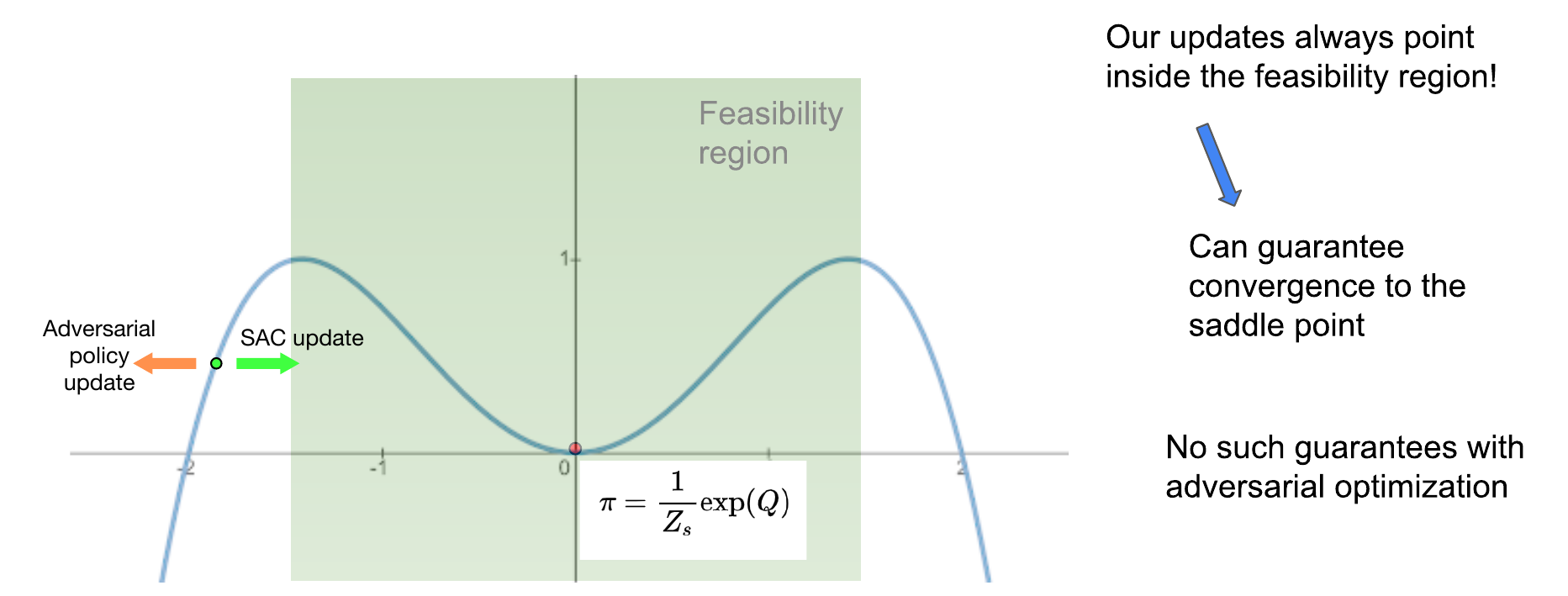}
\vskip -5pt
\caption{\small\textbf{Policy learning.} Comparison of SAC vs adverserial policy update outside the feasibility region.}
\label{fig:update}
\end{figure}

For any $\phi$, our soft actor-critic (SAC) policy update (Sec 4.4) minimizes the KL divergence between the current policy $\pi$ and $\pi_Q$, always pointing towards the the policy minima manifold whereas adversarial policy update relying on the local gradient can diverge away from it (outside the feasibility region). This has the effect, that with sufficient steps, learning with SAC updates is guaranteed to converge to the saddle point, but no such guarantee exists with adversarial policy updates.

\subsection{Effect of various divergences}

In the $Q$-policy space, the policy minima manifold $\pi_Q$ is an energy-based model of $Q$, and doesn't depend on the choice of regularizer $\psi$.

Whereas, the $Q$-maxima manifold is dependent on the choice of regularizer. As the saddle point is formed by the intersection of these two curves (Lemma C.4), we can study how different divergences will affect the saddle point which solves the regularized-IRL problem.

We have that for a choice of $\phi$, the $Q$-maxima manifold is given by the condition:
$$
\phi'(r) \rho_E  - \rho  = 0.
$$
Thus on the maxima manifold, $r = (\phi')^{-1}( \rho/ \rho_E)$. We visualize this in the Fig.~\ref{fig:saddle}, we see that different statistical distances correspond to different saddle points. The overall effect is that that at the saddle point $\pi^*$ remains close to $\pi_E$, but may not be exactly equal as the regularization constrains the policy class.

\begin{figure}[ht]
\centering
\includegraphics[width=0.6\textwidth]{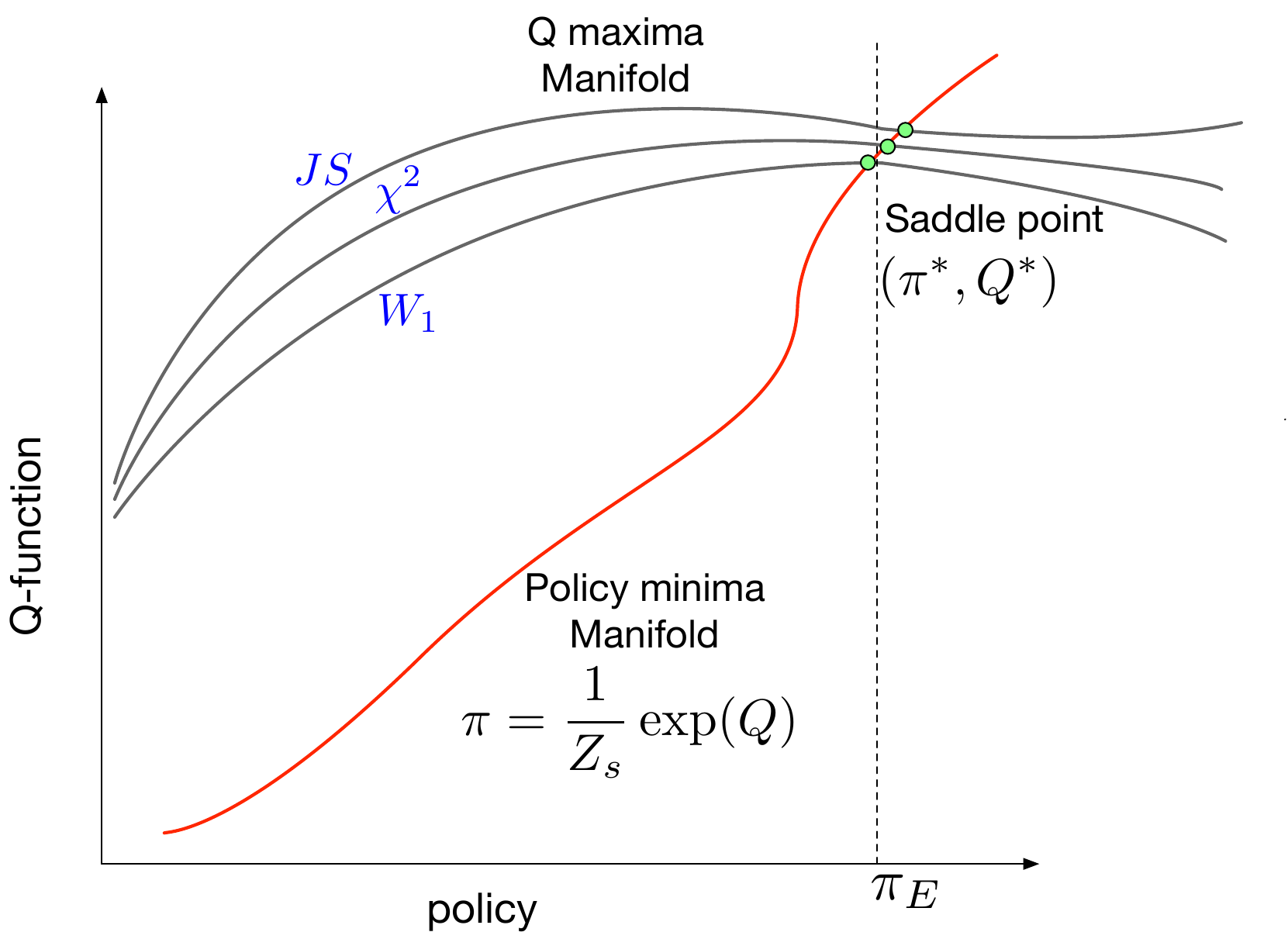}
\vskip -5pt
\caption{\small\textbf{Saddle points.} Effect of regularizer $\psi$ on the saddle point. (not to scale)}
\label{fig:saddle}
\end{figure}

In general, $\pi^*$ is the solution to the (transcendental) equation:
\begin{equation}
\phi'(\mathcal{T}^\pi Q) \rho_E  - \rho_Q  = 0,
\end{equation}
where $\rho_Q$ is the occupancy measure corresponding to $\pi_Q = \frac{1}{Z}\exp{Q}$.

For $f$-divergences, this can be simplified as
\begin{equation}
\mathcal{T}^\pi Q = -f'\left(\frac{\rho_Q}{\rho_E}\right).
\end{equation}

For an IPM parametrized by $\mathcal{F}$, $\phi'(x) = 0$ and the equation will be maximized on the boundary of $\mathcal{F}$, without a closed form equation.

Now, SQIL~\cite{Reddy2020SQILIL} uses the reward of the form $1-0$ dependent on sampling from the expert or policy distributions. This condition corresponds to a maxima manifold in this space, such that instead of the reward being a function of the ratio density of the expert and the policy, it is stochastically dependent on the sampling. Thus, instead of being fixed, the manifold will shift stochastically with the sampling. This has the corresponding effect of shifting the saddle point and can result in numerical instabilities near convergence, as a unique convegence point does not exist for the SQIL style update.

Similary, we can analyze ValueDICE~\cite{Kostrikov2020Imitation}. ValueDICE mimimizes the Reverse-KL divergence between the expert and policy using the Donsker-Varadhan (DV) variational form of Reverse-KL. This corresponds to the maxima manifold with rewards satisfying $r = \log(\rho_E / \rho)$, but suffers from two issues: 1) biased gradient estimates, and 2) adversarial policy updates.

We have already shown how adverserial policy updates are not optimal, we will now focus on fixing the biasing issue with the Reverse-KL distance.

First, the DV representation is given as:
$$
KL(\rho, \rho_E) = \max_{r \in \mathcal{R}} \log  \mathbb{E}_{\rho_E} [e^{-r(s, a)}] - \mathbb{E}_{\rho} [r(s, a)]
$$

This corresponds to  a $\phi(x) = \log  \mathbb{E}_{\rho_E} [e^{-x}]$,
even though its outside the class of $\psi$ we study, it satisfies all the previous properties we developed (Lemma C.1 - C.4).

Now, to unbias the Reverse-KL representation, we propose using the $f$-divergence representation, with $f(t) = t\log t - t + 1$. Then the $f$-divergence for this choice of $f$ is just the Reverse-KL divergence, but it's variational form is:
$$
\max_{r \in \mathcal{R}}  \mathbb{E}_{\rho_E} [-e^{-r(s, a)}] - \mathbb{E}_{\rho} [r(s, a)] - 1
$$ and  corresponds to $\phi(x) = -e^{-x}$ with rewards $r = \log(\rho_E / \rho)$.

Thus, we can obtain the same $Q$-maxima manifold to minimize the Reverse-KL distance as ValueDICE by using this new representation, while avoiding the biasing issue.

\paragraph{Effect of different forms of Reverse-KL}
We test IQ-Learn with different variational representations of Reverse-KL:   Donsker-Varadhan (DV), Original KL (KL), ours Modified KL (KL-fix). We use the LunarLander environment with our offline IL experimental settings and a single expert trajectory. All experiments are repeated over 10 seeds. We show a box-plot of the environment returns for different variational forms and find that our proposed form (KL-fix) and the DV representation perform similarly. The original f-divergence form of KL remains problematic, performing noticeably worse, which may be due to an issue with its corresponding Q-maxima manifold. Compared to DV, our proposed KL variation representation has the advantage of giving unbiased gradient estimates and can be more stable.

\begin{figure}[h]
\vskip -10pt
\centering
\includegraphics[width=0.9\textwidth]{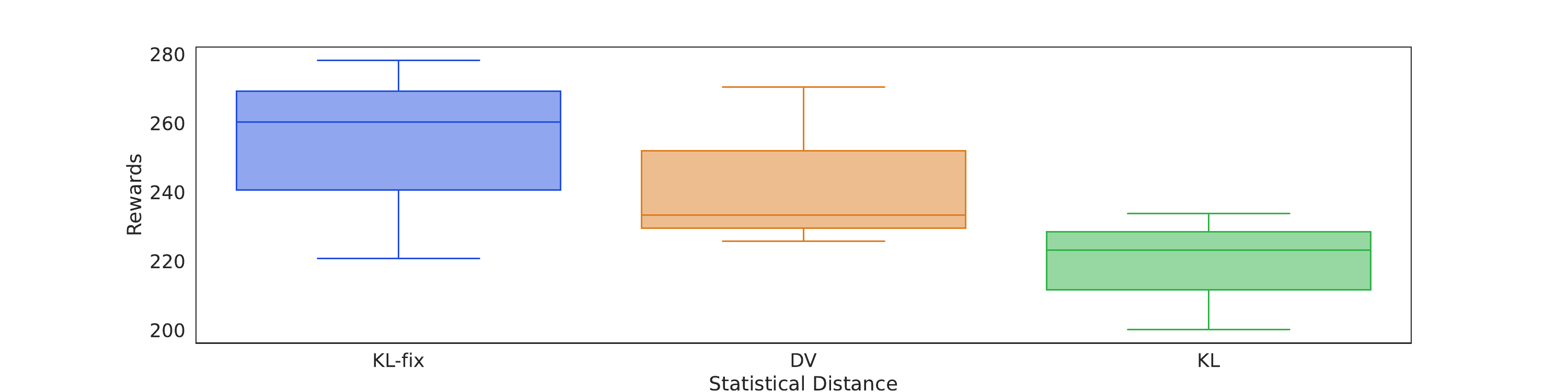}
\vskip -5pt
\caption{\small\textbf{Reverse-KL ablation.} We show environment returns for different variational forms of Reverse-KL on LunarLander.}
\label{fig:div_KL} 
\vskip -10pt
\end{figure}

\subsection{A theoretical guarantee for A.1}
\label{subappx:monotonic}

Motivated by learning a state-dependent reward function, Section~\ref{sec:state_rewards} proposes an alternative objective function for learning $Q$. Here, we provide a form of monotonic improvement guarantee for an idealized version of the objective in Eq.~\ref{eq:state-rewards}. 

Before this, lets introduce some notations that will be useful, as we'll now work with both state and state-action occupancy measures. Let write $\mu\in\Delta_{\states \times \actions}$ an occupancy measure on state-actions, $\rho\in\Delta_{\states}$ an occupancy measure on states. For $\rho \in\Delta_\states$ and $\pi\in \Delta_\actions^\states$, write $\rho \pi \in \Delta_{\states\times\actions}$ defined as $(\rho \pi)(s,a) = \rho_\pi(s)\pi(a|s)$.
Now, let recall Eq.~\ref{eq:state-rewards}:
\begin{equation}
    \J^*(Q) = \E_{s\sim\rho_E}[ \E_{a \sim \text{stop\_grad}(\pi_Q)(\cdot|s)}[\phi((\T^*Q)(s,a))]]  - (1-\gamma) \E_{s\sim\rho_0}[\log \sum_a \exp Q(s,a)].
    \label{eq:Jstar_div}
\end{equation}
Now, we'll consider a more conservative version of this objective function. For a given policy $\pi_k$, define 
\begin{equation*}
    \J_k^*(Q) = \E_{s\sim\rho_E}[ \E_{a \sim \pi_k(\cdot|s)}[\phi((\T^*Q)(s,a))]]  - (1-\gamma) \E_{s\sim\rho_0}[\log \sum_a \exp Q(s,a)].
\end{equation*}
Now, for any initial $Q_0$, define for $k\geq 0$
\begin{equation}
    \begin{cases}
        \pi_k = \pi_{Q_k} = \sm(Q_k)
        \\
        Q_{k+1} = \argmax_Q \J_k^*(Q)
    \end{cases}.
    \label{eq:analysed}
\end{equation}
Eq.~\eqref{eq:Jstar_div} can be see as an optimistic version of Eq.~\eqref{eq:analysed}, in the sense that instead of optimizing to the end each subproblem, we update $\pi_k$ at each gradient step. It is Eq.~\eqref{eq:analysed} that we'll analyse. Before that, we need some assumption: $\phi$ is concave and non-decreasing. Write $f$ the convex conjugate of $-\phi(-x)$ (that is $f^*(x) = - \phi(-x)$), $f$ satisfies $f(1)=0$.


In other words, we restrict ourselves to (a subclass of) $f$-divergences, but it should be possible to adapt the analysis to other cases (eg, IPMs). The core result is the following.

\begin{theorem}
    Under the previous assumption, the sequence of policies $(\pi_k)_{k\geq 0}$ produced by Eq.~\eqref{eq:analysed} satisfies a monotonic improvement guarantee, in the sense that for any $k\geq 0$, we have
    \begin{equation}
        D_f(\rho_{\pi_{k+1}}||\rho_E) - H(\pi_{k+1}) \leq D_f(\rho_{\pi_k}||\rho_E) - H(\pi_k).
    \end{equation}
\end{theorem}
\begin{proof}
    Define for $k\geq 0$
    \begin{align*}
        L_k(\pi, r) &= \E_{s\sim\rho_E}[\E_{a\sim \pi_k(\cdot|s)}[\phi(r(s,a))]] - \E_{s,a\sim\rho_\pi}[r(s,a) - \ln\pi(a|s)]
        \\ \text{and }
        \J_k(\pi, Q) &= L_k(\pi, \mathcal{T}^\pi Q). 
    \end{align*}
    We have that
    \begin{align*}
        \J^*_k(Q) &= \argmin_\pi \J_k(\pi, Q) &\text{(minimized for $Q_{k+1}$)}
        \\ \text{and }
        \pi_{Q_{k+1}} &= \argmin_\pi \max_r L_k(\pi, r)
        \\
        &= \argmin_\pi D_f(\mu_{\pi}|| \rho_E \pi_k) - H(\pi)
        &\text{(by def. of $L_k$).}
    \end{align*}
    This implies that
    \begin{equation*}
        D_f(\mu_{\pi_{k+1}}|| \rho_E \pi_{k}) - H(\pi_{k+1}) \leq D_f(\mu_{\pi_{k}}|| \rho_E \pi_{k}) - H(\pi_{k}).
    \end{equation*}
    We'll work both sides of this bound. For the r.h.s., we have that
    \begin{align*}
        D_f(\mu_{\pi_{k}}|| \rho_E \pi_{k}) &= \sum_{s,a} \rho_E(s)\pi_k(a|s) f\left(\frac{\rho_{\pi_k}(s) \pi_k(a|s)}{\rho_E(s)\pi_k(a|s)}\right)
        \\
        &= \sum_{s} \rho_E(s) \sum_a \pi_k(a|s) f\left(\frac{\rho_{\pi_k}(s)}{\rho_E(s)}\right)
        \\
        &= \sum_s \rho_E(s) f\left(\frac{\rho_{\pi_k}(s)}{\rho_E(s)}\right)
        \\
        &= D_f(\rho_{\pi_k}||\rho_E).
    \end{align*}
    For the l.h.s., we have
    \begin{align*}
        D_f(\mu_{\pi_{k+1}}|| \rho_E \pi_{k}) &= \sum_{s} \rho_E(s) \sum_a \pi_k(a|s) f\left(\frac{\rho_{\pi_{k+1}}(s) \pi_{k+1}(a|s)}{\rho_E(s) \pi_k(a|s)}\right)
        \\
        &\geq  \sum_{s} \rho_E(s) f \left( \sum_a \pi_k(a|s) \frac{\rho_{\pi_{k+1}}(s) \pi_{k+1}(a|s)}{\rho_E(s) \pi_k(a|s)} \right) &\text{(by Jensen)}
        \\
        &= \sum_{s} \rho_E(s) f \left( \frac{\rho_{\pi_{k+1}}(s)}{\rho_E(s)} \right)
        \\
        &= D_f(\rho_{\pi_{k+1}}||\rho_E).
    \end{align*}
    Putting things together, we get
    \begin{align*}
        D_f(\rho_{\pi_{k+1}}||\rho_E) - H(\pi_{k+1}) &\leq 
        D_f(\mu_{\pi_{k+1}}|| \rho_E \pi_{k}) - H(\pi_{k+1}) 
        \\
        &\leq D_f(\mu_{\pi_{k}}|| \rho_E \pi_{k}) - H(\pi_{k})
        \\
        &= D_f(\rho_{\pi_k}||\rho_E) - H(\pi_{k}),
    \end{align*}
    and thus the stated result.
\end{proof}
\section{Appendix D}
\label{appx:D}





\subsection{Implementation Details}
For reproducibility, we release all our expert demonstrations, either trained from scratch or obtained using Stable Baslines3 Zoo~\cite{rl-zoo3}. We also release an efficient expert data generation and data-loading pipeline, that can work with pre-trained Stable Baselines3 models, or arbitary pytorch RL agents. We hope this will make benchmarking for IL easier and help with standardization. Our code is available at \url{https://github.com/Div99/IQ-Learn}.

\subsubsection{Offline Setup}
We mimic prior works' settings \cite{jarrett2020strictly, chan2021scalable} to make our results directly comparable for offline IL.

\paragraph{Expert Demonstrations}
We obtain expert demonstrations by training a DQN~\cite{Mnih2013PlayingAW} agent from scratch for all the environments tested.
Our trajectories were then sub-sampled for every 20th step in Acrobot
and CartPole, and every 5th step in LunarLander.

\paragraph{Training Setup} We test with (1,3,7,10,15) expert trajectories
uniformly sampled from a pool of 1000 expert demonstrations. Each algorithm is trained until
convergence and tested by performing 300 live rollouts in the simulated environment and recording the average episode rewards. We repeat this over 10 seeds, consequently with different initializations and seen trajectories.

\paragraph{Implementation}  All methods use neural networks with the
same architecture of 2 hidden layers of 64 units each connected by exponential linear unit (ELU) activation functions.

We use the original public code implementations of EDM, AVRIL and ValueDICE. Note, ValueDICE is adapted to discrete environments using an actor with  Gumbel-softmax distribution output.

\paragraph{Hyperparameters}
We use batch size $32$ and  $Q$-network learning rate $1e-4$ with  entropy coefficient $0.01$. We found learning rate of $1e-4$ worked best for IQ-Learn on discrete environments. We also found entropy coefficient values $[1e-2, 1e-3]$ to be optimal depending on the environment. Here, we don't use target updates as we found them to give no visible improvement and slow down the training.

\subsubsection{Online Setup}

\paragraph{Expert Demonstrations}
For Mujoco environments, we generate expert demonstrations from scratch using a Pytorch implementation of SAC. For Atari, we generate demonstrations using pre-trained DQN agents from Stable Baselines3 Zoo. For both, we generate a pool of 30 expert demonstrations and sample trajectories uniformly. For Mujoco results, we sample 1 expert demo and for Atari we sample 20 expert demos without any subsampling.

\paragraph{Implementation} For Mujoco, with all methods we use critic and actor networks with an MLP architecture with 2 hidden layers and 256 hidden units, keeping settings similar to original SAC~\cite{Haarnoja2018SoftAO}.  For Atari, with all methods we use a single convolution neural network same as the original DQN architecture~\cite{Mnih2013PlayingAW}. 
For IQ-Learn in continuous environments, for SAC policy updates we sample states from both policy and expert distributions. We regularize policy states in addition to expert states to improve the stability of learning $Q$-values. We use soft target updates and find them helpful for stabilizing the training.  

For BC and GAIL, we use the stable-baselines implementations. For SQIL, we use original public code for Atari environments. For ValueDICE, we use the open-sourced official code.

\paragraph{Hyperparameters}
For SAC style learning, we use default settings of critic learning rate $3e-4$ and policy learning rate values $[3e-4, 3e-5]$. We found $3e-5$ to work well in complex environments and remain stable, although $3e-4$ can be better with simpler environments (like Half-Cheetah). We use a fixed batch size of 256 and found entropy coefficient 0.01 to work well. We use soft target updates with the default SAC smoothing constant $\tau=0.05$. For DQN-style learning on Atari, we use $Q$-network learning rate $1e-4$ with entropy coefficient $1e-4$ and batch size 64. We found entropy coefficient values $[1e-3, 1e-4]$ to work well. We didn't find noticeable improvements with using target updates on Atari (with the exception of Space Invaders, where they stabilize the training).

\subsection{Additional Results}
\paragraph{Mujoco} We show additional results on Mujoco obtained using 10 expert trajectories in Table~\ref{tbl:mujoco10}. We find IQ-Learn gets state-of-art performance in all environments and reaches expert-level rewards.

\begin{table}[h]
    \centering
	\small
	\caption{\small \textbf{Mujoco Results.} We show our performance on MuJoCo control tasks using 10 expert trajectories.}  
	\vskip2pt
	\label{tbl:mujoco10}
  \begin{adjustbox}{max width=\textwidth}
	\begin{tabular}{l|c||c|c|c|c||c}

	Task & Random & BC & GAIL & ValueDICE & IQ (Ours) & Expert \\ \hline
	Hopper & $14 \pm 8$ & $1345 \pm 422$ & $3322 \pm 510$ & $3399 \pm 651$ & $\mathbf{3529 \pm 15}$ & $3533 \pm 39$  \\
    Half-Cheetah & $-282 \pm 80$ & $2701 \pm 950$ & $4280 \pm 1002$ & $4840 \pm 132$  &  $\mathbf{5154 \pm 82}$ & $5098 \pm 62$ \\
	Walker & $1 \pm 5$ & $3730 \pm 1440$ & $4417 \pm 420$ & $4384 \pm 345$ & $\mathbf{5212 \pm 85}$ & $5274 \pm 53$ \\ 
	Ant & $-70 \pm 111$ & $2272 \pm 472$ & $3997 \pm 312$ & $4507 \pm 265$ & $\mathbf{4683 \pm 67}$ & $4700 \pm 80$ \\ 
	Humanoid & $123 \pm 35$ & $2057 \pm 843 $ & $372 \pm 51$ & $2001 \pm 524 $ & $\mathbf{5288 \pm 73}$ & $5313 \pm 210$ \\ \hline
    \end{tabular}
    \end{adjustbox}
    \vskip0pt
\end{table}

\paragraph{Atari Suite.} We show detailed performance of IQ-Learn on Atari Suite environments using 20 expert demonstrations in Table~\ref{tbl:atari_full}. 

\begin{table}[h]
  \centering
   \small
  \caption{\textbf{Results on Atari Suite}. We show our results on Atari Suite tasks using 20 expert demonstrations.}
  \label{tbl:atari_full}
\begin{tabular}{l|c|c}

Env                  &  IQ (Ours) & Expert \\ \hline
Pong    & $19 \pm 2$   &     $21 \pm 0$         \\
Breakout    & $320 \pm 72$  & $376 \pm 34$   \\
Space Invaders & $807 \pm 102$  & $823 \pm 272$ \\
BeamRider & $3025 \pm 845$  & $4295 \pm 1173 $ \\ 
Seaquest & $2349 \pm 342$  & $2393 \pm 291$ \\ 
Qbert & $12940 \pm 2026$  & $11496 \pm 1988$ \\ \hline
\end{tabular}
\end{table}

\paragraph{Reward Correlations.} We show the Pearson correlation coefficient of our learnt rewards with environment rewards in  Table~\ref{tbl:reward_corr}.

\begin{table}[h]
  \centering
  \small
  \caption{\textbf{Reward Correlations}. We show pearson correlations between our learnt reward and the env rewards. }
  \label{tbl:reward_corr}
\begin{tabular}{l|c}

Env                  &  Reward correlation\\ \hline
Cartpole & 0.99 \\
LunarLander & 0.92 \\
Hopper & 0.99 \\
Half-Cheetah & 0.86 \\
Pong    & 0.67    \\ \hline
\end{tabular}
\vskip -5pt
\end{table}

\begin{wrapfigure}{r}{0.5\textwidth}
\vskip -25pt
\centering
\includegraphics[width=\linewidth]{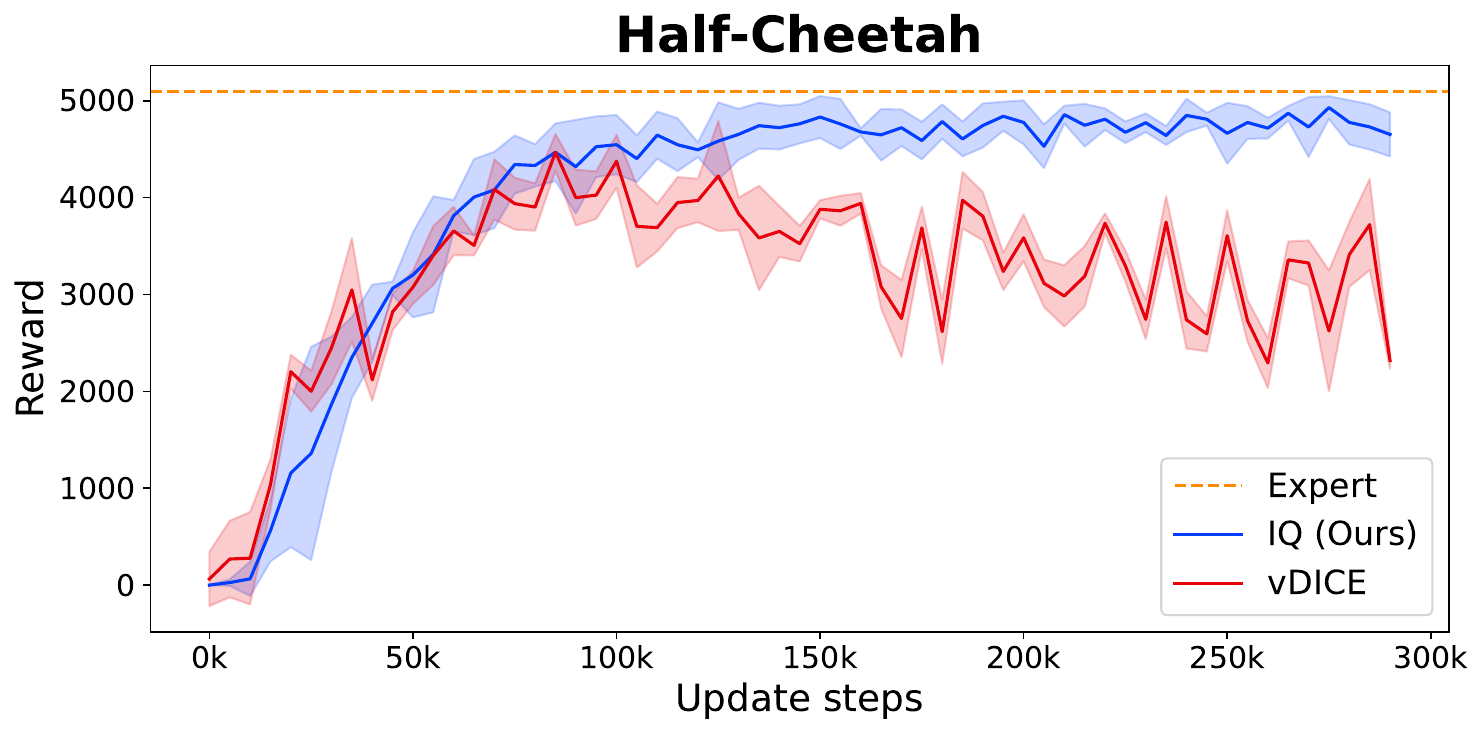}
\vskip -7pt
\caption{\small{\textbf{Half-Cheetah overfitting comparision}}}
\label{fig:cheetah}
\vskip -30pt
\end{wrapfigure}
\paragraph{Do we overfit?} Compared to ValueDICE, we don't observe overfitting using IQ-Learn with the number of update steps. We show a comparision on Half-Cheetah environment using one expert trajectory in Fig~\ref{fig:cheetah}. ValueDICE begins to overfit around 100k update steps, whereas IQ-Learn converges to expert rewards and remains stable.


\subsection{Recovering Rewards}
\begin{wrapfigure}{r}{0.35\textwidth}
\vskip -20pt
\centering
\includegraphics[width=\linewidth]{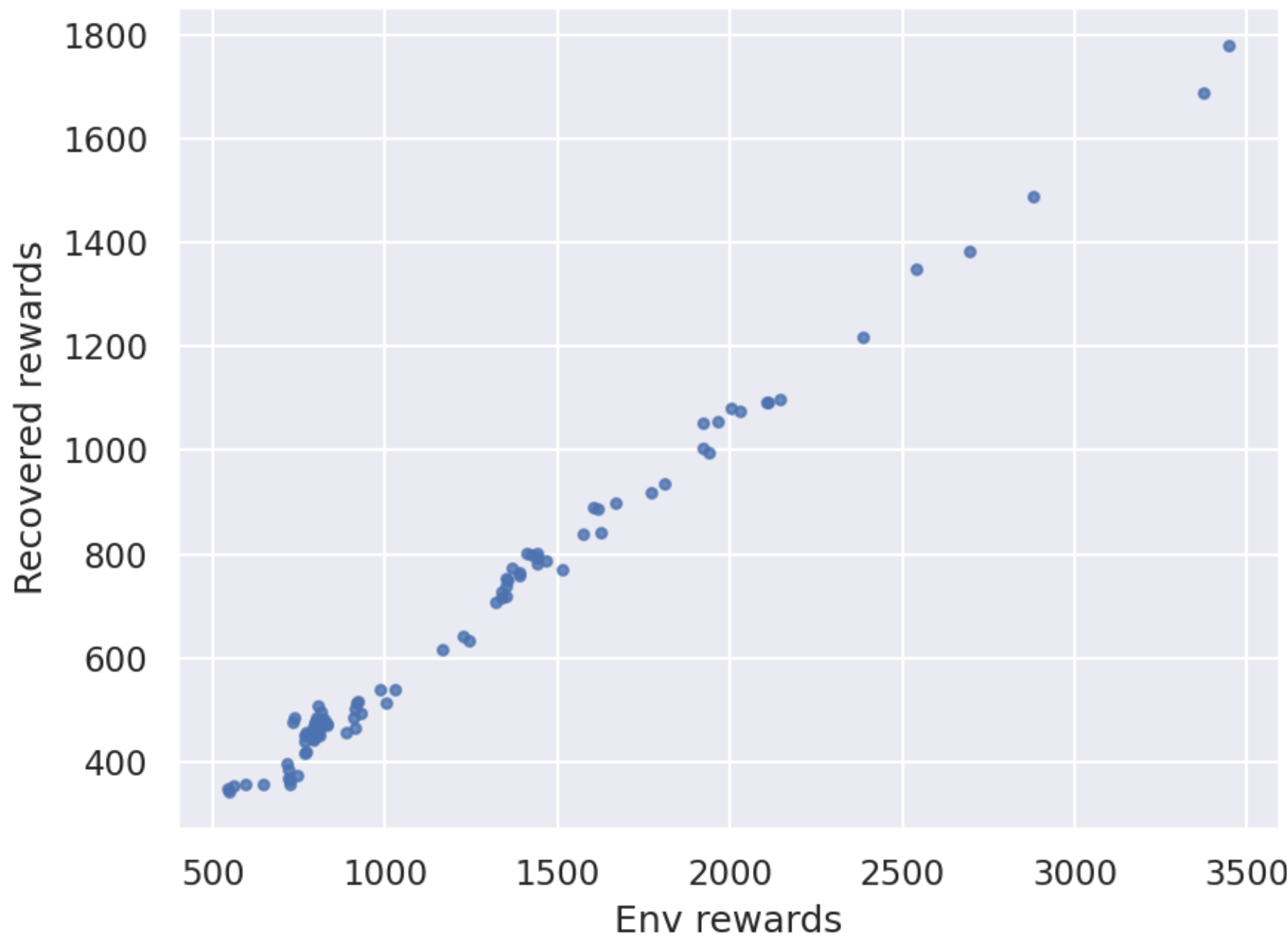}
\vskip -5pt
\caption{\small\textbf{Hopper correlations}}
\label{fig:corr} 
\vskip -20pt
\end{wrapfigure}

 We show visualizations of our reward correlations on the Hopper environment using 10 expert demonstrations in Fig~\ref{fig:corr}. We obtain a Pearson correlation of $0.99$ of our recovered episode rewards compared with the original environment rewards, showing that our rewards are almost linear with the actual rewards, and thus can be used for Inverse RL. Note, that to recover rewards with IQ-Learn, we need to sample the current state and the next state.
 
 We perform similar comparisons on GAIL and SQIL, obtaining Pearson coefficients of $0.90$ and $0.72$ respectively.

In the main paper, we also show recovered rewards on a simple grid environment by using sampling based $Q$-learning with a simple $Q$-network having two hidden layers. In the section below, we further compare IQ-Learn on a tabular setting. 

\paragraph{Tabular Inverse RL} 
To further validate IQ-Learn as a method for IRL and show we recover correct rewards, we directly compare with the classical Max Entropy IRL~\cite{Ziebart2008MaximumEI} method on a tabular Grid world setting, by using an open-source implementation\footnote{\url{https://github.com/yrlu/irl-imitation}}. We implement IQ-Learning as a modification to tabular value iteration. The classical method requires repeated backward and forward passes, to calculate soft-values and action probabilities for a given reward and optimize the rewards respectively. IQ-Learn skips the expensive backward pass and directly optimizes the rewards. We show comparision in Fig~\ref{fig:rewards}, where we find our method recovers very similar rewards while being more than 3x faster.

\begin{figure}[t]
\vskip -10pt
\centering
\includegraphics[width=\linewidth]{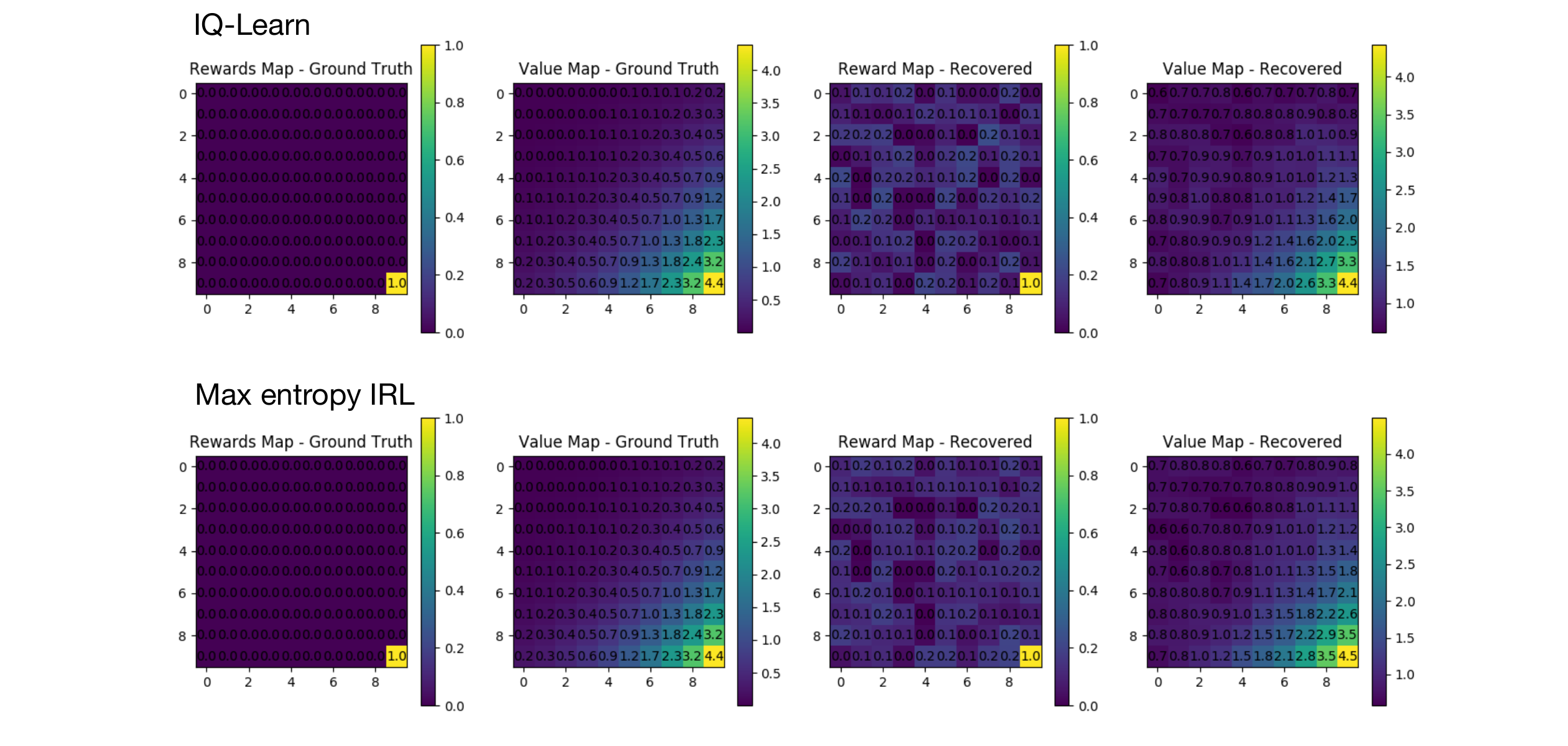}
\vskip -5pt
\caption{\small\textbf{Tabular Grid Rewards.} We recover similar rewards as Max entropy IRL (Ziebert et al.) while avoiding an expensive backward pass.}
\label{fig:rewards} 
\vskip -15pt
\end{figure}
\subsection{Imitation learning with Observations}
\begin{wraptable}{R}{0.5\linewidth}
  \centering
  \vskip -15pt
  \caption{\textbf{Results on ILO}. We show evironment returns using 1 and 10 expert demonstrations.}
  \label{tbl:ilo}
\begin{tabular}{l|c|c}

Env                  & 1 demo & 10 demos \\ \hline
CartPole    & $452 \pm 50$   &     $485 \pm 25$         \\
LunarLander    & $20 \pm 102$  & $220 \pm 69$   \\
Hopper & $2507 \pm 345$  & $3465 \pm 51$ \\ \hline
\end{tabular}
\vskip -10pt
\end{wraptable}
We show results for IQ-Learn trained with using only expert observations in Table~\ref{tbl:ilo}. We test on CartPole, LunarLander and Hopper environments with 1 and 10 expert demonstrations using online IL settings without any subsampling of trajectories. We find that with one expert demonstration, we get below expert-level rewards, and as expected, our performance suffers compared to with using expert actions. We find using 10 demonstrations is enough to reach expert-level performance in these simple environments.

Target updates are helpful in stabilizing the training in this setting.

\section{Appendix E}
\subsection{Dynamics-Aware Imitation Learning and the Loop MDP}
In this section we illustrate the importance of dynamics-awareness in imitation learning with a toy MDP based on the \texttt{Loop} MDP from~\cite{ross2010efficient}.
The MDP is shown in Fig~\ref{fig:loop-MDP}. The MDP has a fixed length of 100 steps. The key problem for dynamics-unaware algorithms, such as behavioural cloning, is the behaviour in state $s_2$. If we happen to use an expert trajectory where the expert never visits state $s_2$, then the learned policy will not necessarily have the right behaviour in state $s_2$. This is because the objective for behavioural cloning is to match the action probabilities in the expert states, and $s_2$ is not in the expert states visited. However, the dynamics-aware methods are able to deduce that taking action $a_1$ in state $s_2$ will return the imitator to state $s_1$. Although this MDP is simple, it illustrates a general advantage of dynamics-aware methods which will hold in many situations. In particular, it will hold for environments where the expert may keep very close to an optimal trajectory, yet it is possible to recover back to that trajectory if a small mistake is made, such as in autonomous lane-keeping in a car.
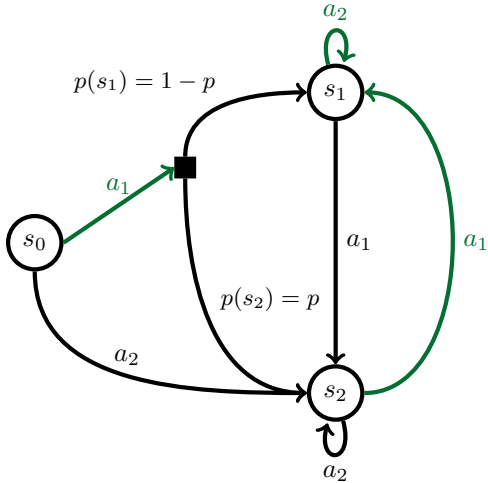
\begin{figure}[H]
\centering
 \begin{tikzpicture}
   \node[circle, ultra thick] (S0) at (-4,0) [draw, minimum width=0.5cm,minimum height=0.5cm] {$s_0$};
   \node[rectangle, ultra thick] (S0A1) at (-2,1) [draw, fill=black, minimum width=0.05cm,minimum height=0.05cm] {};
  \node[circle, ultra thick] (S1) at (0,2) [draw, minimum width=0.5cm,minimum height=0.5cm] {$s_1$};
  \node[circle, ultra thick] (S2) at (0,-2) [draw, minimum width=0.5cm,minimum height=0.5cm] {$s_2$};
  \draw [->, darkgreen, ultra thick] (S0.east) to node[above]{$a_1$} (S0A1.west);
  \draw [->, ultra thick] (S0.south) to [out=270, in=180] node[above]{$a_2$} (S2.west);
  \draw [->, ultra thick] (S0A1.north) to [out=90, in=180] node[above left]{{\footnotesize $p(s_1)=1 - p$}}(S1.west);
  \draw [->, ultra thick] (S0A1.south) to [out=270, in=180] node[above right]{{\footnotesize $p(s_2)=p$}}(S2.west);
    \draw [->, darkgreen, ultra thick] (S1) edge[loop above] node[above]{$a_2$} (S1);
    \draw [->, ultra thick] (S1) to node[right]{$a_1$} (S2);
    \draw [->, ultra thick] (S2) edge[loop below] node[below]{$a_2$} (S2);
    \draw [->,darkgreen, ultra thick] (S2.east) to [out=0,in=0] node[right]{$a_1$}(S1.east);
  \end{tikzpicture}
    \caption{A variant of the \texttt{Loop} MDP from \cite{ross2010efficient}. Taking actions labelled in {\color{darkgreen} green} gives $1$ reward, while actions in black give reward $0$. The MDP is stochastic for action $a_1$ in state $s_0$, which with probability $p$ leads to state $s_2$, and with probability $1-p$ leads to state $s_1$. 
 }\label{fig:loop-MDP}
\end{figure}%
To substantiate this illustrative case, we implemented this MDP and evaluated a few methods. We use a single expert trajectory which goes from $s_0$ to $s_1$, never going to state $s_2$. We set $p=0.5$ for this experiment. The results are in Table~\ref{tab:loop-mdp-results}, averaged over five random seeds. They are as we expect, with the dynamics-aware methods able to convincingly master the environment and find the optimal policy, while the behavioural cloning approach achieves around 50 reward. This is because it learns the wrong behaviour in state $s_2$ so gets zero reward in that state in the 50\% of the time that taking action $a_1$ results in a transition to state $s_2$. 

\begin{table}[h]
  \centering
 \caption{Results of imitation learning algorithms on the \texttt{Loop} MDP described above. We observe that the dynamics-unaware behavioural cloning baseline performs much worse than the other dynamics-aware methods.}
\label{tab:loop-mdp-results}
\begin{tabular}{@{}ll@{}}
\toprule
Method                  & Episode Reward \\ \midrule
Behavioural Cloning     & $54 \pm 5$               \\
SQIL                    & $100 \pm 0$               \\
IQ (Online, $\chi^2$) & $100 \pm 0$               \\ \bottomrule
\end{tabular}
\end{table}

\subsection{Ablation on Gamma}

\begin{wrapfigure}{r}{0.55\textwidth}
\vskip -40pt
\centering
\includegraphics[width=\linewidth]{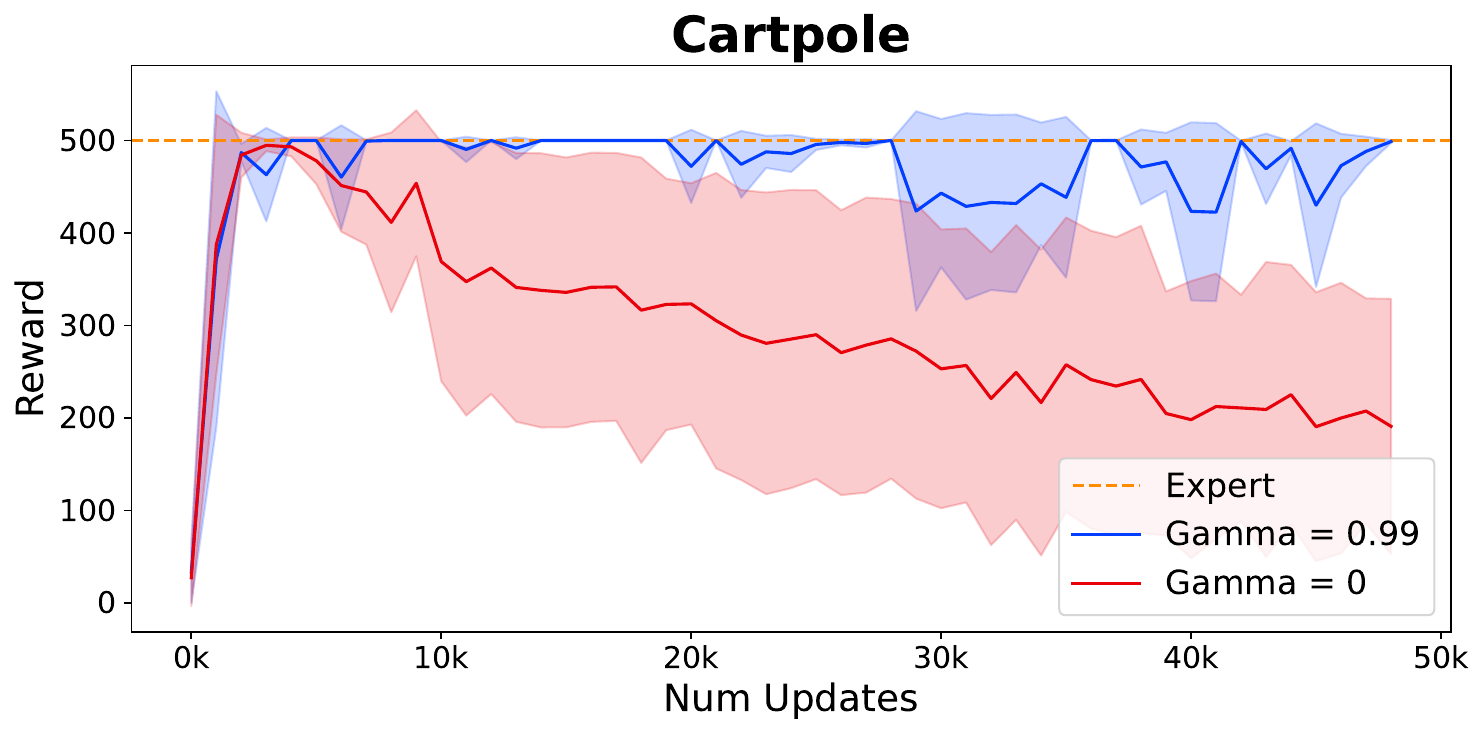}
\vskip -5pt
\caption{\small\textbf{Ablation on Gamma}}
\label{fig:gamma} 
\vskip -20pt
\end{wrapfigure}

The dynamics are encoded in our learning objective by the discount factor $\gamma$, and setting it to zero removes dynamic-awareness in IQ-Learn.

To show how dynamics help with learning, we do an ablation on $\gamma$ with IQ-Learn. We use the offline IL settings for CartPole environment with one expert trajectory.

We set $\gamma$ to $0.99$ and $0$. The results are visualized in Fig~\ref{fig:gamma}, we can see that without the dynamics the training is not stable and there is a strong decay in the rewards obtained by the IL agent from the environment. Whereas, when using dynamics, we see that the training is stable and properly converges.
\section{Appendix F}
\label{appx:F}
\subsection{Generalization over distribution shift}

We show our method can be robust to distribution shifts between the expert and policy and perform additional experiments over two different settings: 1) Initial distribution shift using a modified LunarLander env motivated by \cite{Reddy2020SQILIL} and 2) Goal distribution shift using DeepMind Control Suite.

\subsubsection{Initial distribution shift}

We experiment with initial shift distribution in the LunarLander-v2 environment similar to \cite{Reddy2020SQILIL}. The agent is typically initialized in a small zone at the middle top of the screen. Instead, we modify the environment to initialize the agent near the top-left corner of the screen. We use experts from the unmodified environment, and test whether the agent can still learn to land the lunar lander while recovering from the initial distribution shift.

\textbf{Offline Case}: We find in the offline case that the agent cannot learn to recover from the occupancy shift. The lander typically tends to fly off the frame and shows random behavior. This is expected as IQ-learn is not aware of the shift of initial distributions between the agent and the expert, and can’t explore the environment to correct the initial state shift to match the occupancy distributions.

\textbf{Online Case}: In the online case, we find that the agent can sufficiently explore the environment, and learns a behavior of first horizontally moving the lander from the top left to the top center and then successfully imitating the original expert trajectory, receiving an avg. episode reward of ~250 with 10 expert demos.

An extra consideration here is in Eq. 9, where we originally only apply reward regularization to the expert states, but we find applying regularization to both expert and policy states to be beneficial in this case. As it enforces the learning of an implicit reward function that can generalize outside the expert distribution to more arbitrary policy states.

\subsubsection{Goal Distribution Shift}

We experiment with the \textit{reacher\_easy} task in DeepMind Control Suite. We choose the reacher environment as it is a multi-task environment, where the goal given by the target position changes in every episode randomly. Such environments have been found to be very difficult to solve using IRL \cite{Yu2019MetaInverseRL} as a large number of expert demos are needed to fully cover the goal distributions, and usually require meta-IRL methods to figure the right task context for a given expert demonstration like PEMIRL \cite{Yu2019MetaInverseRL}.

We test with different number of expert demonstrations: $(1, 5, 10, 20)$ each with different target positions on the offline and online settings. The average expert performance is $\sim 990$ in this case and we report averaged results over $100$ episodes with different targets.

\begin{wraptable}{R}{0.5\linewidth}
  \centering
  \vskip -12pt
  \caption{\textbf{Offline}. We show evironment returns vs number of experts on \textit{reacher\_easy} for offline case.}
  \label{tbl:offline_reacher}
\begin{tabular}{l|c}

Num Experts  & Rewards \\ \hline
1 & 105.4 \\
5 & 120.1 \\
10 & 210.6 \\
20 & 325.0 \\
\hline
\end{tabular}
\vskip -5pt
\end{wraptable}

\textbf{Offline Case}: In the offline setting, a single demonstration is typically not enough to learn a generalized reward function and leads to a reward that overfits to a particular target position. We quantify the results in Table~\ref{tbl:offline_reacher}, with the observation that imitation learning performance improves with the number of expert demos. This can be justified, as more experts with different targets allow learning a reward function that is better generalizable.

\begin{wraptable}{R}{0.5\linewidth}
  \centering
  \vskip -1pt
  \caption{\textbf{Online}. We show evironment returns vs number of experts  on \textit{reacher\_easy} for online case.}
  \label{tbl:online_reacher}
\begin{tabular}{l|c}

Num Experts  & Rewards \\ \hline
1 &	271.3 \\
5 &	485.1 \\
10 & 545.0 \\
20 & 734.9 \\
50 & 926.1 \\
\hline
\end{tabular}
\vskip -5pt
\end{wraptable}

\textbf{Online Case}: In the online setting, our method is able to explore the environment over different episodes and can learn to correct the behavior leading to better performance. In particular, given a sufficient number of expert demos, it can learn to associate what expert behavior to imitate given a particular target and learns a more reward function generalizable over multiple goals. We show quantitative results in Table~\ref{tbl:online_reacher}.\\


BC and GAIL on \textit{reacher\_easy}  even with $50$ experts obtain mean rewards of $325.2$ and $440.1$ respectively, which is equivalent to what we see using our method with just $5$ expert demos! It is surprising to us that our method can learn a reward to figure out what goal state to reach, acting as a \textbf{meta-learner} even when not engineered specifically to do so.


\end{document}